\documentclass[12pt]{iopart}

\usepackage{hyperref}

\expandafter\let\csname equation*\endcsname\relax
\expandafter\let\csname endequation*\endcsname\relax

\usepackage{epsfig}
\usepackage{graphicx}
\usepackage{amsmath}
\usepackage{amssymb}

\usepackage{bm}
\usepackage{amsthm}
\usepackage{algorithm}
\usepackage{algorithmic}

\usepackage{color,soul}
\usepackage{subcaption}

\graphicspath{{./figures/}}

\newcommand{\R}{\mathbb{R}}
\newcommand{\N}{\mathbb{N}}
\newcommand{\norm}[1]{\left\vert\left\vert#1\right\vert\right\vert}
\newcommand{\cP}[1]{\mathcal{P}(#1)}
\newcommand{\cM}[1]{\mathcal{M}(#1)}
\newcommand{\cK}{\mathcal{K}}
\newcommand{\diff}[2]{\frac{\text{d}#1}{\text{d}#2}}
\newcommand{\dx}[1]{\text{d}#1}
\newcommand{\ip}[2]{\left\langle#1,#2\right\rangle}
\newcommand{\dup}[2]{\left(#1,#2\right)}

\newcommand{\abs}[1]{\left\vert#1\right\vert}

\DeclareMathOperator{\dom}{dom}
\DeclareMathOperator{\argmin}{argmin}

\DeclareMathOperator{\argmax}{argmax}
\DeclareMathOperator{\interior}{int}
\DeclareMathOperator{\prox}{prox}

\newtheorem{theorem}{Theorem}
\newtheorem{lemma}{Lemma}
\newtheorem{remark}{Remark}

\begin{document}
\title[]{The Maximum Entropy on the Mean Method for Image Deblurring}

\author{Gabriel Rioux$^1$, Rustum Choksi$^1$, Tim Hoheisel$^1$, Pierre Mar\'{e}chal$^2$, and Christopher Scarvelis$^1$}

\address{$^1$Department of Mathematics and Statistics, McGill University, Montreal, QC H3G 1Y6, Canada\\
$^2$ Institut de Math\'ematiques de Toulouse, Universit\'e Paul Sabatier, Toulouse 31 062, France}
\ead{$\{$gabriel.rioux,christopher.scarvelis$\}$@mail.mcgill.ca,\\$\{$rustum.choksi,tim.hoheisel$\}$@mcgill.ca, \\ pierre.marechal@math.univ-toulouse.fr } 
\vspace{10pt}

\begin{abstract}
Image deblurring is a notoriously challenging ill-posed inverse problem. In recent years, a wide variety of approaches have been proposed based upon regularization at the level of the image or on techniques from machine learning. In this article, we adapt the principal of maximum entropy on the mean (MEM) to both deconvolution of general images and point spread function (PSF) estimation (blind deblurring). This 
approach  shifts the paradigm towards regularization at 
the level of the probability distribution on the space of images whose expectation is 
our estimate of the ground truth. We present a self-contained analysis of this method, reducing the problem to solving a differentiable, strongly convex finite-dimensional optimization problem for which there exists an abundance of black-box solvers.  
The strength of the MEM method lies in its simplicity, its ability to handle large blurs, and its potential for generalization and modifications.  
When images are embedded with symbology (a known pattern), we show how our method can be applied to approximate the unknown blur kernel with remarkable effects. \end{abstract}

%
\vspace{2pc}
%
%
%
%

\section{Introduction}

Ill-posed inverse problems permeate the fields of image processing and machine learning. 
Prototypical examples stem from non-blind (deconvolution) and blind deblurring of digital images. 
The vast majority of methods for image deblurring are based on some notion of regularization  at the image level. In this article, we present, analyse and test a different method known in information theory as {\it Maximum Entropy on the Mean} (MEM).

The general idea of {\it maximum entropy} dates back to Jaynes
in 1957 (\cite{jaynes1957information1,jaynes1957information2}). A vast literature originates in Jaynes'
paper, on conceptual and theoretical aspects, but also
on the multiple applications of the principle of
maximum entropy. Based upon these ideas,  Navaza and others developed
a method for solving ill-posed problems in crystallography (\cite{Navaza85, Navaza86,Dacunha90, Decarreau92}). 
Further applications  were made to deconvolution problems in 
astrophysics\footnote{The setting of astronomical imaging is ideal for deconvolution, as in most situations an accurate estimation of the point spread function can be obtained. Indeed, one can estimate the point spread function by calibrating the telescope using reference stars or by analysing the properties of the optics system for example \cite{Nagy98,Soulez13}.} (\cite{skilling1984maximum, Narayan86, Besnerais91, Urban96}).

In image reconstruction (as in other fields of applied
sciences), it is necessary to distinguish between
the method of maximum entropy (ME) and the principle of
Maximum Entropy on the Mean (MEM) as described
in~\cite{gamboa1989methode,Dacunha90}.
In the former, the gray level of each pixel
is interpreted as a probability; the image
is then normalized, so that the values add up to one,
and the maximum entropy principle is applied under the
available constraints (see~\cite{skilling1984maximum} and
the references therein). In the latter (see for example \cite{Heinrich96, Marechal97, Besnerais99}), the pixelated image
is considered as a random vector, and the principle of the
maximum entropy is applied to the probability distribution
of this vector; the available constraints are imposed on the
expectation under the unknown probability distribution, and thus become
constraints on the probability of the image
(see~\cite{le1991aperture,le1999new,Marechal97}).

A definite advantage of the MEM is that, by moving to an
upper level object (a probability distribution on the object
to be recovered), it becomes possible to incorporate nonlinear
constraints via the introduction of a prior distribution.
This results in a very flexible machinery, which makes the MEM
very attractive for a wide variety of applications.

To the best of our knowledge, the first occurrence of 
the MEM inference scheme appeared in~\cite{shore1981minimum},
for the purpose of spectral analysis.
Surprisingly, the MEM has not been widely used for
image deconvolution. Indeed, citations from the key articles
suggest that MEM is not well-known in the image processing\footnote{We remark that Noll \cite{Noll} did implement a MEM framework for deblurring a few photo images but with only modest results.}
and machine learning communities and, to our knowledge,
its full potential has not been explored and implemented
in the context of blind deblurring in image processing. One
of the goals of this article is to rectify this by demonstrating
that a reformulation of the MEM method can produce a general
scheme for deconvolution and in certain cases, kernel estimation,
which compares favourably with the state of the art and is
amenable to very large blurs. Moreover, its ability to elegantly
incorporate prior information opens the door to many
possible avenues for generalizations and other applications.

\medskip

Following the presentation in Le Besnerais, Bercher, and Demoment \cite{Besnerais99}, let us 
first state the classical MEM approach introduced by Navaza (\cite{Navaza85, Navaza86}) for solving linear inverse problems.  
We work with vectors in $\R^d$, where in the context of an image,  $d$ represents the number of pixels. 
Given a $d \times d$ convolution matrix ${C}$ and an observable data ${z}\in \R^d$, we wish to determine the ground truth  ${x} \in \R^d$
and noise ${n} \in \R^d$ where 
\begin{equation}\label{oldMEM-0}
  { z} \, = \, {C}{x} + {n}  \qquad\quad  {\rm or} \qquad \quad {z} \, = \, {H} {y}, \quad {\rm where}\,\,\,  {H} = [ {C}, { I}] \quad {\rm and} \quad {y} \, = \, 
 \begin{bmatrix}  { x}\\ { n}  \end{bmatrix}.
 \end{equation}
If $\rho,\mu$ are two probability measures on a subset $\Omega\subset \R^d$, we define 
the Kullback-Leibler divergence to be 
\begin{equation}
  \mathcal{K}(\rho,\mu)=
  \left\{
  \begin{array}{ll}
    \int_{\Omega}\log\left( \diff{\rho}{\mu} \right)\dx{\rho}, &\qquad  \rho,\mu\in\cP{\Omega},\,\rho\ll\mu,\\
    +\infty, &\qquad  \text{ otherwise,}
  \end{array}
  \right.
  \label{eq:KL}
\end{equation}
where $\cP{\Omega}$ denotes the space of probability measures on $\Omega$. 
One can think of $\rho$ as the probability distribution associated with ${x}$ and $\mu$ as some prior distribution. 
In the classical MEM approach, the best estimate of ${y}$ is taken to be 
\[ 
 \widehat{y}  \, := \, \mathbb{E}_{\bar{\rho}} [{X}], \quad {\rm 
 where}\,\,\, \bar{\rho} \, = \, {\argmin}_{\rho}  \left\{\mathcal{K}(\rho,\mu)\} \,\, \big\vert \,\, \, {H}  \mathbb{E}_\rho [{y}] = {z} \right\}.
\]
This infinite-dimensional variational problem is recast as follows: 

\[\text{setting } \,\,\, F({y}) \, : = \, \min_{\rho} \left\{ \mathcal{K}(\rho,\mu) \,\, \big\vert \,\, \mathbb{E}_{\rho} [X] = {y} \right\},  \]
\begin{equation}\label{oldMEM-1}
\text{solve }\,\,\,  \min_{y}  F({y})  \,\,\,\,\,  {\rm s.t} \,\, \,\,\, {H} {y} = {z}.\end{equation}
To solve (\ref{oldMEM-1}), they introduce what they call the {\it log-Laplace transform} of $\mu$ which turns out to be the convex conjugate of $F$ {under some assumptions on $\mu$}:
\begin{equation}\label{oldMEM-2}
  F^\ast ({s}) \, := \, \log \int \exp \langle {s}, {u}\rangle  \, \dx{\mu ({u})}.  
\end{equation}
Then, via Fenchel-Rockafellar duality, they show that the primal problem (\ref{oldMEM-1}) has a dual formulation 
\[  \min_{{H} {y} = {z}}  F({y}) \, = \, \max_{\lambda \in \R^d} \,  D (\lambda) \, : = \, \lambda^T {z} \, - \, F^\ast ({H}^T  \lambda). 
 \]
The primal-dual recovery formula states that if $\widehat{\lambda}$ is a maximizer of $D$ and $\widehat{{s}}   : = {H}^T \widehat{\lambda}$ then the primal solution  is  
\begin{equation}\label{oldMEM-3} \widehat{y} \, = \, \nabla F^\ast (\widehat{{s}}).\end{equation}

\medskip

In this article, we present a general MEM based deconvolution approach which does not directly include a noise component, but rather treats the constraint via an additive fidelity term. We also treat directly the infinite-dimensional primal problem over probability densities. 
A version of this approach was recently implemented by us 
(cf. \cite{Riouxetal}) for the blind deblurring of 2D QR and 1D UPC barcodes, where we showed its remarkable ability to blindly deblur highly blurred and noisy barcodes. 
In this article, we present a self-contained, short (yet complete) analysis of the general theory including a stability analysis, and then apply it to both non-blind deblurring and kernel (PSF) estimation (blind deblurring). 
Our computational results are dramatic and compare well with the state of the art. 

Let us provide a few details of our approach to MEM as a method for deblurring.
In our notation, pixel values are taken from a compact subset of $\Omega \subset \R^d$, $\rho$ the  distribution of the image and the prior $\mu$ are both in ${\cP{\Omega}}$, and ${b}\in \R^d$ is a measured signal representing the blurred and possibly  noisy image. 
Our best guess of the ground truth is determined via 
\[ \bar{x} \, : = \, \mathbb{E}_{\bar{\rho}}[{X}],\,\,\,\,\,\,  \hbox{\rm  where} \,\,  \bar{\rho} \, := \, 
 \underset{\rho\in\cP{\Omega}}{\arg\min}  \left\{ \mathcal{K}(\rho,\mu)+\frac{\alpha}{2}\norm{b-C\mathbb{E}_{\rho}[{X}]}_2^2\right\}.\]
Here $\alpha>0$ is the fidelity parameter, directly linked to the fidelity of $C\mathbb{E}_{\rho}[{X}]$ to the blurred image ${b}$
(see (\ref{fid-par})).  
This primal problem has a finite-dimension dual (cf. Theorem \ref{thm-maindual}) and a recovery formula for the minimizer ${\bar{\rho}}$ in terms of the optimizer $\bar{\lambda}$ of the dual problem. 
To compute the expectation of ${\bar{\rho}}$, we use the moment-generating function $\mathbb{M}_{{X}}[t]$ of the 
 prior $\mu\in\cP{\Omega}$ to show in Section \ref{sec-probint} that
 \begin{equation}\label{oldMEM-5}
  \mathbb{E}_{\bar{\rho}}[{X}]=\left.\nabla_t\log\left(\mathbb{M}_{{X}}[t]\right)\right\vert_{C^{T}\bar{\lambda}}.
\end{equation}
This is the analogous statement to (\ref{oldMEM-3}). 
When the moment generating function of $\mu$ is known, finding our estimate of the ground has thus been reduced to the optimization in $d$ dimensions of an explicit strongly convex, differentiable  function. 

Of course, effective implementation of deconvolution is directly tied to the presence of noise. In our MEM formulation, we do not directly include noise as a (solved for) variable.   
We thus make three important remarks about noise: 
\begin{itemize}
\item In Theorem \ref{thm-stability} we prove a stability estimate valid for any prior $\mu$ which supports the fact that our method is stable with respect to small amounts of noise. 
\item With only a very modest prior (with a sole purpose of imposing box constraints), we demonstrate that 
for moderate to large amounts of noise, our method works well by preconditioning with expected patch log likelihood (EPLL) denoising \cite{Zoran11}. 
\item While the inclusion of noise-based priors will be addressed elsewhere, we briefly comment on directly denoising with our MEM deconvolution method in Section \ref{sec-pror}. 
\end{itemize}

Our MEM deconvolution scheme  can equally well be used for PSF estimation, hence blind deblurring. Indeed, 
given some approximation of the ground truth image, say $\tilde x$, we estimate the ground truth PSF $\bar{c}$ as 
\begin{equation}\label{oldMEM-4}
 \bar{c} \, : = \,   \mathbb{E}_{\bar{\eta}}[{X}],\,\,\,\,\,\,  \hbox{\rm  where} \,\,  \bar{\eta}\, := \, 
  \underset{\eta\in\cP{\Omega}}{\arg\min}  \left\{\mathcal{K}(\eta,\nu) + \frac{\gamma}{2}\norm{\mathbb{E}_{\eta}[{X}]*\tilde{x}- {b}}^2_2\right\}.\end{equation}
Many blind deblurring algorithms iterate between estimating the image $x$ and the PSF $c$. One could take a similar iterative approach by invoking prior information about $c$ in  $\nu$. 
However, here we proceed via  symbology; that is, we assume the image contains a known pattern analogous to a finder pattern in a QR or UPC barcode (cf. \cite{Riouxetal}).  
In these cases, one focuses (\ref{oldMEM-4}) on $\tilde x$, the part of the image which is known. While our numerical results  (cf. Figures  \ref{fig:BlindCam}, \ref{fig:BlindText}, \ref{fig:BlindPlane}) are dramatic, the presence and exploitation of a finder pattern is indeed restrictive and does weaken our notion of {\it blind} deblurring. 
Moreover, it renders comparisons with other methods unfair. Nevertheless, we feel that, besides synthetic 
images like barcodes, there are many possible applications wherein some part of an image is {\it a priori} known.
As a matter of fact, let us stress that the ability
of the MEM to incorporate nonlinear constraints (via the
prior measure) may tremendously improve the estimation
of the convolution kernel. For example, the choice of the support
of the prior will enable us to confine the kernel to any prescribed
closed convex subset.  In the paper, we exploit very little of 
this potential yet already obtain remarkable reconstructions.

We reiterate  that our approach directly treats the infinite-dimensional primal problem over probability
measures 
$\rho$ and its finite-dimensional dual, a setting often referred to as {\it partially finite programming} \cite{Borwein92}. 
While we can also reformulate our primal problem in terms of a finite-dimensional analogue of (\ref{oldMEM-1}) defined at the image level 
(see Section \ref{sec-reform}), an advantage to our approach lies in the fact that one may be interested in computing the optimal $\bar{\rho}$ or a general linear functional ${\mathcal L}$
of $\bar{\rho}$ for which a simple closed form formula (analogous to (\ref{oldMEM-3}) and (\ref{oldMEM-5}) for the expectation) fails to exist. 
Here one could adapt a stochastic gradient descent algorithm to compute 
${\mathcal L}\bar{\rho}$ wherein the only requirement on the prior would be that one can efficiently sample from $\mu$.

The paper is organized as follows. We start by briefly recalling current methods for deblurring, highlighting the common theme of regularization at the level of the image, or level of the PSF (for blind deblurring). 
After a brief section on preliminaries,  we  present  in Sections \ref{sec-MEM1}, \ref{sec-MEM2}, and  \ref{sec-probint} our MEM deconvolution method which we  have just outlined. 
We then address PSF estimations via the exploitation of symbology in Section \ref{sec-MEM3}. We prove a stability estimate in Section \ref{sec-stability}. 
Numerical examples with comparisons for both deconvolution and blind deblurring with symbology are presented in Section \ref{sec-results}. We present a short discussion of the role of the prior and future work in Section \ref{sec-pror}. Finally we end with a short conclusion, highlighting the MEM method and the novelties of this article.

\subsection{Current Methods}

The process of capturing one channel of a blurred image $b\in\R^{n\times m}$ from a ground truth channel $x\in\R^{n\times m}$ is modelled throughout by the relation $b=c*x$,
where $*$ denotes the 2-dimensional convolution between the kernel $c\in\R^{k\times k}$ ($k<n,m$) and the ground truth; this model represents spatially invariant blurring. For images composed of more than one channel, blurring is assumed to act on a per-channel basis. Therefore, we derive a method to deblur one channel and apply it to each channel separately.

Most current blind deblurring methods consist of solving 
\begin{equation}
  \inf_{\substack{x\in\R^{n\times m}\\c\in\R^{k\times k}}}\left\{R(x,c)+\frac{\alpha}{2}\norm{c*x-b}^2_2\right\},
  \label{eq:stdapp}
\end{equation}
where $R:\R^{n\times m}\times\R^{k\times k}\to\R$ serves as a regularizer which permits the imposition of certain constraints on the optimizers and $\alpha>0$ is a fidelity parameter. 
This idea of regularization to solve ill-posed inverse problems  dates back to Tikhonov \cite{Tik}.
The inherent difficulties in solving \eqref{eq:stdapp} depend on the choice of regularizer. A common first step is to iterately solve the following two subproblems, 
\[
  \tilde{c}\in\inf_{c\in\R^{k\times k}}\left\{ R(\tilde{x},c)
   +\frac{\alpha}{2}\norm{c*\tilde{x}-b}^2_2 \right\},
\]
\[
  \tilde{x}\in\inf_{x\in\R^{n\times m}}\left\{ R(x,\tilde{c}) +\frac{\alpha}{2}\norm{\tilde{c}*x-b}^2_2 \right\},
\]
the first subproblem is a kernel estimation and the second is a deconvolution. We will utilize this approach in the sequel. 

Taking $R_c=\norm{\cdot}_2^2$ and replacing $\tilde{x}$ and $b$ by their derivatives to estimate the kernel has yielded good results in different methods (see e.g. \cite{Cho09,Liu19,Pan14,Pan17,Yan17}) and can be efficiently solved using spectral methods. The deconvolution step often involves more elaborate regularizers which may be non-differentiable (such as the $\ell_1$ regularizer) or non-convex (such as the $\ell_0$ regularizer), thus it is often the case that new optimization techniques must be developped to solve the resulting problem. For example, the half-quadratic splitting method \cite{Xu11} has been used to great effect when the $\ell_0$ regularizer is employed \cite{Liu19,Pan14,Pan17,Yan17}. Optimization methods that are well-suited for the $\ell_1$ regularizer include primal-dual interior-point methods \cite{Kim07}, iterative shrinkage thresholding algorithms \cite{Beck09,Daub04}, and the split Bregman method \cite{Goldstein09}. In our case, standard convex optimization software can be used regardless of the choice of prior, as the problem to be solved is strongly convex and differentiable.

Approaches that are not based on machine learning differ mostly in the choice of regularizer; common choices include $\ell_2,\ell_1,$ and $\ell_0$-regularization (and combinations thereof), which promote sparsity of the solution to varying degrees (see \cite{Cho09,Joshi08,Levin07,You96,Yuan07},\cite{Blomgren97,Chan98,Dobson96,Perrone16,Rudin94,Shan08,Wang08,Xu10,You962}, and \cite{Liu19,Pan14,PanE14,Pan17} for some examples which employ, respectively, $\ell_2$, $\ell_1$ and $\ell_0$ regularizers on the image or its derivatives).

Such regularizers may be employed to enforce sparsity on other quantities related to the image. For example, one can employ $\ell_0$-regularization of the dark or bright channel to promote sparsity of a channel consisting of local minima or maxima in the intensity channel \cite{Pan16,Yan17}. Further, one can regularize the coeffients of some representation of the image. In particular, many natural images have sparse representations in certain transformation domains. Some frames which have seen use in deblurring and deconvolution include wavelets \cite{Banham96,Elad10,Figueiredo03}, DCT (Discrete Cosine Transform) \cite{Ng99},  contourlets \cite{Tzeng10}, and framelets \cite{Cai09}. Moreover, combinations of different frames to exploit their individual properties can also be utilized \cite{Cai092,Fadili06,Starck03}.

Other notable choices of regularizer include the ratio of the $l_1$ and $l_2$ norm to enforce sparsity \cite{Krishnan11}, the weighted nuclear norm, which ensures that the image or gradient matrices have low rank \cite{Ren16,Yair18}, a penalization based on text-specific properties \cite{Cho12}, the spatially-variant hyper-Laplacian penalty \cite{Cheng19,Shi16}, a surface-aware prior which penalizes the surface area of the image \cite{Liu19}, and approximations of the $\ell_0$ penalty via concatenation of a quadratic penalty and a constant \cite{Xu13} or via a logarithmic prior \cite{Perrone15}.

Approaches based on machine learning include modelling the optimization problem as a deep neural network \cite{Schuler15,Wu20} and estimating the ground truth image from a blurred input without estimating the kernel using convolutional neural networks (CNNs) \cite{Nah17,Noroozi17}, recurrent neural networks (RNNs) \cite{Tao18} or generative adversarial networks (GANs)\cite{Kupyn17,Ramakrishnan17}. Other popular methods consist of estimating the kernel via CNNs and using it to perform deconvolution \cite{Sun15}, deconvolving via inversion in the Fourier domain and denoising the result using a neural network \cite{Dabov08,Danielyan12,Schuler13}, and 
learning dictionaries for sparse representations of natural images \cite{Dong13,Dong11,Hu10,Mairal08,Nimisha17,Xu17}.

\section{Preliminaries}
We begin by recalling some standard definitions and establishing notation.
We refer to \cite{Zalinescu02} for convex analysis in infinite dimensions and \cite{Rockafellar70} for the finite-dimensional setting. We follow \cite{Rudin87} as a standard reference for real analysis.

Letting $(X,\tau)$ be a separated locally convex space, we denote by $X^*$ its topological dual. The duality pairing between $X$ and its dual will be written as $(\cdot,\cdot):X\times X^*\to\R$ in order to distinguish it from the canonical inner product on $\R^d$, $\ip{\cdot}{\cdot}:\R^d\times\R^d\to\R$. 
For $f:X\to\bar{\R}\equiv\R\cup\{-\infty,+\infty\}$, an extended real-valued function on $X$, the (Fenchel) conjugate of $f$ is $f^*:X^*\to\bar{\R}$ defined by
  \[
    f^*(x^*)=\sup_{x\in X}\left\{ (x,x^*)-f(x) \right\},
  \]
  using the convention $a-(-\infty)=+\infty$ and $a-(+\infty)=-\infty$ for every $a\in\R$.
  The subdifferential of $f$ at $\bar{x}\in X$ is the set
  \[
    \partial f(\bar{x})=\left\{x^*\in X^*|(x-\bar{x},x^*)\leq f(x)-f(\bar{x})\;\forall x\in X\right\}.
  \]
We define $\dom f:=\left\{x\in X|f(x)< +\infty  \right\}$, the domain of $f$, and say that $f$ is proper if $\dom f\neq\emptyset$ and $f(x)>-\infty$ for every $x\in X$. $f$ is said to be lower semicontinuous if $f^{-1}([-\infty,\alpha])$ is $\tau$-closed for every $\alpha\in\R$. 
 
  A proper function $f$ is convex provided for every $x,y\in \dom f$ and $\lambda\in (0,1)$,
\[
  f(\lambda x+(1-\lambda)y)\leq \lambda f(x)+(1-\lambda)f(y).
\]
If the above inequality is strict whenever $x\neq y$, $f$ is said to be strictly convex. If $f$ is proper and for every $x,y\in \dom f$ and $\lambda \in (0,1)$,
\[
  f(\lambda x+(1-\lambda)y)+\lambda(1-\lambda)\frac{c}{2}\norm{x-y}^2\leq \lambda f(x)+(1-\lambda)f(y),
\]
then $f$ is called $c$-strongly convex.

For any set $A\subseteq X$, the indicator function of $A$ is given by 
\[
  \delta_A:X\to\R\cup\{+\infty\},\qquad x\mapsto\left\{
    \begin{array}{ll}
      0,&x\in A,\\
      +\infty, & \text{otherwise.}
    \end{array}
  \right.
\]

For any $\Omega\subseteq\R^d$, we denote by $\cP{\Omega}$ the set of probability measures on $\Omega$.
The set of all signed Borel measures with finite total variation on $\Omega$ will be denoted by $\cM{\Omega}$. We say that a measure is $\sigma$-finite (on $\Omega$) if $\Omega=\cup_{i\in\N}\Omega_i$ with $\abs{\mu(\Omega_i)}< +\infty$.

Let $\mu$ be a positive $\sigma$-finite Borel measure on $\Omega$ and $\rho$ be an arbitrary Borel measure on $\Omega$, we write $\rho\ll\mu$ to signify that $\rho$ is absolutely continuous with respect to $\mu$, i.e. if $A\subseteq \Omega$ is such that $\mu(A)=0$, then $\rho(A)=0$. If $\rho\ll\mu$ there exists a unique function $\diff{\rho}{\mu}\in L^1(\mu)$ for which \[
  \rho(A)=\int_A \diff{\rho}{\mu}\, \dx{\mu},\qquad \text{$\forall\;A\subseteq\Omega$ measurable}.
\]
The function $\diff{\rho}{\mu}$ is known as the Radon-Nikodym derivative (cf. \cite[Thm. 6.10]{Rudin87}).
These measure-theoretic notions were used previously to define the Kullback-Leibler divergence \eqref{eq:KL}. 

For $\Omega\subseteq\R^d$, $\eta\in\cM{\Omega}$ we denote, by a slight abuse of notation, $\mathbb{E}_{\eta}[{X}]$ to be a vector whose $k^{\text{th}}$ component is $(\mathbb{E}_{\eta}[{X}])_{k}=\int_{\Omega}x_k\dx{\eta(x)}$. Thus, $\mathbb{E}_{(\cdot)}[{X}]$ is a map from $\cM{\Omega}$ to $\R^d$ whose restriction to $\cP{\Omega}$ is known as the expectation of the random vector ${X}=[X_1,\cdots,X_d]$ associated with the input measure.

Finally, the smallest (resp. largest) singular value $\sigma_{\min}(C)$ (resp. $\sigma_{\max}(C)$) of the matrix $C\in\R^{m\times n}$ is the square root of the smallest (resp. largest) eigenvalue of $C^TC$.
\section{The MEM Method}
\subsection{Kullback-Leibler Regularized Deconvolution and the Maximum Entropy on the Mean Framework}\label{sec-MEM1}  
{\bf Notation:}
We first establish some notation pertaining to deconvolution.
The convolution operator $c*$ will be denoted by the matrix $C:\R^d\to\R^d$ acting on a vectorized image $x\in\R^d$ for $d=nm$ and resulting in a vectorized blurred image for which the $k^{th}$ coordinate in $\R^d$ corresponds to the $k^{th}$ pixel of the image. We assume throughout that the matrix $C$ is nonsingular.

\medskip



We recall that traditional deconvolution software functions by solving \eqref{eq:stdapp} with a fixed convolution kernel $c$. Our approach differs from previous work by adopting the maximum entropy on the mean framework which posits that the state best describing a system is given by the mean of the probability distribution which maximizes some measure of entropy \cite{jaynes1957information1,jaynes1957information2}. As such, taking $\Omega\subseteq\R^d$ to be compact, $\mu\in\cP{\Omega}$ to be a prior measure  and $b\in\R^d$ to be a blurred image, our approach is to determine the solution of 

\begin{equation}
  \inf_{\rho\in\cP{\Omega}}\left\{ \mathcal{K}(\rho,\mu)+\frac{\alpha}{2}\norm{b-C\mathbb{E}_{\rho}[{X}]}_2^2\right\}=\inf_{\cP{\Omega}}\left\{ f+g\circ A \right\},
  \label{eq:Primal}
\end{equation}
for
\begin{equation}
  f=\mathcal{K}(\cdot,\mu),\quad g=\frac{\alpha}{2}\norm{b+(\cdot)}_2^2,\quad A=-C\mathbb{E}_{(\cdot)}[{X}].
\label{eq:fg}
\end{equation}
The following lemma establishes some basic properties of $f$.
\begin{lemma}
  The functional $f:\cM{\Omega}\to\bar{\R}$ is proper, weak$^*$ lower semicontinuous and strictly convex.  
  \label{lem:KLlem}
\end{lemma}
\begin{proof}
  We begin with strict convexity of $f$. Let $x\in\Omega$ and $t\in(0,1)$ be arbitrary moreover let $\rho_1\neq\rho_2$ be elements of $\cP{\Omega}$ and $\rho_t=t\rho_1+(1-t)\rho_2$. We have
\begin{align*}
  \log\left( \frac{\diff{\rho_t}{\mu}(x)}{t+(1-t)} \right)\diff{\rho_t}{\mu}(x)&=\log\left( \frac{t\diff{\rho_1}{\mu}(x)+(1-t)\diff{\rho_2}{\mu}(x)}{t+(1-t)} \right)\left(t\diff{\rho_1}{\mu}(x)+(1-t)\diff{\rho_2}{\mu}(x) \right)\\
  &\leq t\log\left( \diff{\rho_1}{\mu}(x)\right)\diff{\rho_1}{\mu}(x)+(1-t)\log\left( \diff{\rho_2}{\mu}(x)\right)\diff{\rho_2}{\mu}(x).
\end{align*} 
The inequality is due to the log-sum inequality \cite[Thm. 2.7.1]{Cover06}, and since $\rho_1\neq \rho_2$, $\diff{\rho_1}{\mu}$ and $\diff{\rho_2}{\mu}$ differ on a set $E\subseteq\Omega$ such that $\mu(E)>0$. The strict log-sum inequality  therefore implies that the inequality is strict for every $x\in E$. 
Since integration preserves strict inequalities,
\[
  f(\rho_t)=\int_{\Omega\backslash E}\log\left( \diff{\rho_t}{\mu} \right)\diff{\rho_t}{\mu}\dx{\mu}+\int_{E}\log\left( \diff{\rho_t}{\mu} \right)\diff{\rho_t}{\mu}\dx{\mu}<tf(\rho_1)+(1-t)f(\rho_2)
\]
so $f$ is, indeed, strictly convex.

It is well known that the restriction of $f$ to $\cP{\Omega}$ is weak$^*$ lower semicontinuous and proper (cf. \cite[Thm. 3.2.17]{Deuschel89}). Since $f\equiv +\infty$ on $\cM{\Omega}\backslash \cP{\Omega}$, $f$ preserves these properties.
\end{proof}


Problem \eqref{eq:Primal} is an infinite-dimensional optimization problem with no obvious solution and is thus intractable in its current form. However, existence and uniqueness of solutions thereof is established in the following remark.
\begin{remark} 
  First, the objective function in \eqref{eq:Primal} is proper, strictly convex and weak$^*$ lower semicontinuous since $f$ is proper, strictly convex and weak$^*$ lower semicontinuous whereas $g\circ A$ is proper, weak$^*$ continuous and convex.
 
  Now, recall that the Riesz representation theorem \cite[Cor. 7.18]{Folland99} identifies $\cM{\Omega}$ as being isomorphic to the dual space of $(C(\Omega),\norm{\cdot}_{\infty})$. Hence, by the Banach-Alaoglu theorem, \cite[Thm. 5.18]{Folland99} the unit ball of $\cM{\Omega}$ in the norm-induced topology\footnote{The norm here is given by the total variation, we make precise that the weak$^*$ topology will be the only topology considered in the sequel.} ($\mathbb{B}^*$)  is weak$^*$-compact. 

  Since $\dom f\subseteq \cP{\Omega}\subseteq\mathbb{B}^*$, standard theory for the existence of minimizers of $\tau$-lower semicontinuous functionals on $\tau$-compact sets \cite[Cor. 3.2.3]{Butazzo14} imply that \eqref{eq:Primal} has a solution and strict convexity of $f$ guarantees that it is unique.
%
\end{remark}
Even with this theoretical guarantee, direct computation of solutions to \eqref{eq:Primal} remains infeasible.
In the sequel, a corresponding finite-dimensional dual problem will be established which will, along with a method to recover the expectation of solutions of \eqref{eq:Primal} from solutions of this dual problem, permit an efficient and accurate estimation of the original image.
\subsection{Dual Problem}\label{sec-MEM2}

In order to derive the (Fenchel-Rockafellar) dual problem to \eqref{eq:Primal}  we provide the reader with the Fenchel-Rockafellar duality theorem in a form expedient for our study, cf. e.g.  \cite[Cor. 2.8.5]{Zalinescu02}.

\begin{theorem}[Fenchel-Rockafellar Duality Theorem]
  Let $(X,\tau)$ and $(Y,\tau')$ be locally convex spaces and let $X^*$ and $Y^*$ denote their dual spaces. Moreover, let $f:X\to\R\cup\{+\infty\}$ and $g:Y\to\R\cup\{+\infty\}$ be convex, lower semicontinuous (in their respective topologies) 
  and proper functions and let $A$ be a continuous linear operator from $X$ to $Y$. Assume that there exists $\bar{y}\in A\dom f\cap \dom g$ such that $g$ is continuous at $\bar{y}$. Then 
  \begin{equation}    
  \inf_{x\in X}\left\{ f(x)+g(-Ax) \right\}=\max_{y^*\in Y^*}\left\{-f^*(A^*y^*)-g^*(y^*)  \right\}
  \label{eq:FRdualform}
\end{equation}
with $A^*$ denoting the adjoint of $A$. Moreover, $\bar{x}$ is optimal in the primal problem if and only if there exists $\bar{y}^*\in Y^*$ satisfying $A^*\bar{y}^*\in\partial f(\bar{x})$ and $\bar{y}^*\in\partial g(-A\bar{x})$.

\label{thm:FR}
\end{theorem}

In \eqref{eq:FRdualform}, the minimization problem is referred to as the primal problem, whereas the maximization problem is called the dual problem. Under certain conditions, a solution to the primal problem can be obtained from a solution to the dual problem. 

\begin{remark}[Primal-Dual Recovery]
  In the context of Theorem \ref{thm:FR}, $f^*$ and $g^*$ are proper, lower semicontinuous and convex, also $(f^*)^*=f$ and $(g^*)^*=g$ \cite[Thm. 2.3.3]{Zalinescu02}. Suppose additionally that $0\in\interior(A^*\dom g^*-\dom f^*)$. 
  
  Let $\bar{y}^*\in\argmax_{Y^*} \left\{-f^*\circ A^*-g^*\right\}$. By the first order optimality conditions, \cite[Thm. 2.5.7]{Zalinescu02} 
\[
  0\in\partial \left( f^*\circ A^*+g \right)(\bar{y}^*)=A\partial f^*(A^*\bar{y}^*)+\partial g(y^*),
\]
the second expression is due to \cite[Thm. 2.168]{Bonnans00} (the conditions to apply this theorem are satisfied by assumption). Consequently, there exists $\bar{z}\in\partial g^*(\bar{y}^*)$ and $\bar{x}\in\partial f^*(A^*\bar{y}^*)$ for which $\bar{z}=-A\bar{x}$. Since $f$ and $g$ are proper, lower semicontinuous and convex we have \cite[Thm. 2.4.2 (iii)]{Zalinescu02}:
\[
  A^*\bar{y}^*\in\partial f(\bar{x}),\qquad \bar{y}^*\in\partial g(\bar{z})=\partial g(-A\bar{x}).
\] Thus Theorem \ref{thm:FR} demonstrates that $\bar{x}$ is a solution of the primal problem, that is if $\bar{y}^*$ is a solution of the dual problem, $\partial f^*(A^*\bar{y}^*)$ contains a solution to the primal problem.

If, additionally, $f^*(A^*\bar{y}^*)< +\infty$ \cite[Prop. 2.118]{Bonnans00} implies that, 
  \begin{equation} 
    \bar{x}\in\partial f^*(A^*\bar{y}^*) \, =\, \underset{x\in X}{\arg\max}\left\{ \dup{x}{A^*\bar{y}^*}-f(x) \right\}.
  \label{eq:PrimalDualRec}
\end{equation}
We refer to \eqref{eq:PrimalDualRec} as the primal-dual recovery formula.
\label{rem:primdual}
\end{remark}

A particularly useful case of this theorem is when $A$ is an operator between an infinite-dimensional locally convex space $X$ and $\R^d$, as the dual problem will be a finite-dimensional maximization problem. Moreover, the primal-dual recovery  is easy  if $f^*$ is G{\^ a}teaux differentiable at $A^*\bar{y}^*$, in which case the subdifferential and the derivative coincide at this point \cite[Cor. 2.4.10]{Zalinescu02}, so \eqref{eq:PrimalDualRec} reads $\bar{x}=\nabla f^*(A^*\bar{y}^*)$. Some remarks are in order to justify the use of this theorem.

\begin{remark}
  It is clear that $\cP{\Omega}$ endowed with any topology is not a locally convex space, however it is a subset of $\cM{\Omega}$. Previously, $\cM{\Omega}$ was identified with the dual of $(C(\Omega),\norm{\cdot}_{\infty})$, thus the dual of $\cM{\Omega}$ endowed with its weak$^*$ topology $(\cM{\Omega},w^*)^*$ can be identified with $C(\Omega)$ \cite[Thm. 1.3]{Conway07} with  duality pairing $(\phi,\rho)\in C(\Omega)\times\cM{\Omega}\mapsto\int_{\Omega}\phi\dx{\rho}$.

  Since $\dom f\subseteq \cP{\Omega}$, the $\inf$ in \eqref{eq:Primal} can be taken over $\cM{\Omega}$ or $\cP{\Omega}$ interchangeably.
  \label{rmk:Embed}
\end{remark}


In the following we verify that $A$ is a bounded linear operator and compute its adjoint.

\begin{lemma}
  The operator $A:\cM{\Omega}\to\R^d$ in \eqref{eq:fg} is linear and weak$^*$ continuous. Moreover, its adjoint is the mapping $z\in\R^d\mapsto\ip{C^Tz}{\cdot}\in C(\Omega)$.
\label{lem:adjoint}
\end{lemma}
\begin{proof} 

  We begin by demonstrating weak$^*$ continuity of $\mathbb{E}_{(\cdot)}[{X}]:\cM{\Omega}\to\R^d$. Letting $\pi_i:\R^d\to\R$ denote the projection of a vector onto its $i-$th coodinate, we have
  \begin{equation}
\mathbb{E}_{\rho}[{X}]= \left( (\pi_1,\rho),\dots,(\pi_n,\rho) \right)
    \label{eq:proj}
  \end{equation}
  Thus, $A$ is the composition of a weak$^*$ continuous operator from $\cM{\Omega}$ to $\R^d$ and a continuous operator from $\R^d$ to $\R^d$ and hence is weak$^*$ continuous.



  Eq. \eqref{eq:proj} equally establishes linearity of $A$, since the duality pairing is a bilinear form.

  The adjoint can be determined by noting that
  \[
    \ip{\mathbb{E}_{\rho}[{X}]}{z}=\sum_{i=1}^d\int_{\Omega}x_i\dx{\rho(x)}z_i=\int_{\Omega}\sum_{i=1}^dx_iz_i\dx{\rho(x)}=(\ip{z}{\cdot},\rho),
  \]
  so,
  \[
    \ip{C\mathbb{E}_{\rho}[{X}]}{z}=\ip{\mathbb{E}_{\rho}[{X}]}{C^Tz}=(\ip{C^Tz}{\cdot},\rho),
  \]
  yielding $A^*(z)=\ip{C^Tz}{\cdot}$.
\end{proof}

We now compute the conjugates of $f$ and $g$, respectively and provide an explicit form for the dual problem of \eqref{eq:Primal}. 

\begin{lemma}
  The conjugate of $f$ in \eqref{eq:fg} is $f^*:\phi\in C(\Omega)\mapsto\log\left( \int_{\Omega}\exp(\phi)\dx{\mu} \right)$. In particular, $f^*$ is finite-valued. Moreover, for any $\phi\in C(\Omega)$, $\argmax_{\cP{\Omega}}\left\{ (\phi,\cdot)-\cK(\cdot,\mu) \right\}=\left\{ \bar{\rho}_{\phi} \right\}$, the unique probability measure on $\Omega$ for which 
\begin{equation}
  \diff{\bar{\rho}_{\phi}}{\mu}=\frac{\exp\phi}{\int_{\Omega}\exp\phi\, \dx{\mu}}.
  \label{eq:argmax}
\end{equation}
\label{lem:KLconj}
\end{lemma}
\begin{proof}
  We proceed by direct computation:
  \begin{align*}
    f^*(\phi)&=\sup_{\rho\in\cM{\Omega}}\left\{ (\phi,\rho)-\cK(\rho,\mu)\right\}\\
    &=\sup_{\rho\in\cP{\Omega}}\left\{ (\phi,\rho)-\cK(\rho,\mu) \right\}\\
    &=\sup_{\rho\in\cP{\Omega}}\left\{\int_{\Omega} \log\left( \frac{\exp\phi}{\diff{\rho}{\mu}} \right)\dx{\rho} \right\},
  \end{align*}
  where we have used the fact that $\dom f\subseteq\cP{\Omega}$ as noted in Remark \ref{rmk:Embed}.
  Note that $\exp\phi\in C(\Omega)\subseteq L^1(\rho)$ and since $t\mapsto \log t$ is concave, Jensen's inequality \cite[Thm. 3.3]{Rudin87} yields
  \begin{equation}
    f^*(\phi)\leq\sup_{\rho\in\cP{\Omega}}\left\{\log\left(\int_{\Omega}\frac{\exp\phi}{\diff{\rho}{\mu}}\, \dx{\rho}\right)\right\}=\log\left(\int_{\Omega}\exp\phi\, \dx{\mu}  \right)    
    \label{eq:bound}
  \end{equation}
  Letting $\bar{\rho}_{\phi}$ be the measure with Radon-Nikodym derivative
  \[
    \diff{\bar{\rho}_{\phi}}{\mu}=\frac{\exp\phi}{\int_{\Omega}\exp\phi\, \dx{\mu}},
  \]
  one has that 
  \[
    (\phi,\bar{\rho}_{\phi})-\cK(\bar{\rho}_{\phi},\mu)=(\phi,\bar{\rho}_{\phi})-\int_{\Omega}\log\left(\frac{\exp\phi}{\int_{\Omega}\exp\phi\dx{\mu}}  \right)\dx{\bar{\rho}_{\phi}}=\log\left( \int_{\Omega}\exp\phi\, \dx{\mu} \right),
  \]
  so $\bar{\rho}_{\phi}\in\argmax_{\cP{\Omega}}\left\{(\phi,\cdot)-\cK(\cdot,\mu)  \right\}$ as $\bar{\rho}_{\phi}$ saturates the upper bound for $f^*(\phi)$ established in \eqref{eq:bound}, thus $f^*(\phi)=\log\left(\int_{\Omega}\exp\phi\, \dx{\mu}  \right)$. Moreover $\bar{\rho}_{\phi}$ is the unique maximizer since the objective is strictly concave.

  With this expression in hand, we show that $\dom f^*=C(\Omega)$. To this effect, let $\phi\in C(\Omega)$ be arbitrary and note that,
  \[
    \exp\left( \phi(x) \right)\leq\exp\left( \max_{\Omega}\phi \right),\qquad(x\in\Omega).
  \]
  Thus,
  \[
    f^*(\phi)=\log\left(\int_{\Omega}\exp\phi\dx{\mu}\right)\leq\log\left( \exp\left(\max_{\Omega}\phi  \right) \right)=\max_{\Omega}\phi<+\infty,
  \]
  since $C(\Omega)=C_b(\Omega)$ by compactness of $\Omega$. Since $\phi$ is arbitrary, $\dom f^*=C(\Omega)$ and, coupled with the fact that $f^*$ is proper \cite[Thm. 2.3.3]{Zalinescu02}, we obtain that $f^*$ is finite-valued.
\end{proof}

\begin{lemma}
  The conjugate of $g$ from \eqref{eq:fg} is $g^*:z\in\R^d\mapsto \frac{1}{2\alpha}\norm{z}_2^2-\ip{b}{z}$.
\label{lem:normconj}
\end{lemma}
\begin{proof}
  The assertion follows from the fact that $\frac{1}{2}\norm{\cdot}_2^2$ is self-conjugate \cite[Ex. 11.11]{Rockafellar09} and some standard properties of conjugacy \cite[Eqn. 11(3)]{Rockafellar09}. 
\end{proof}

Combining these results we obtain the main duality theorem.

\begin{theorem}\label{thm-maindual}
  The (Fenchel-Rockafellar) dual of \eqref{eq:Primal} is given by 
  \begin{equation}
    \max_{\lambda\in\R^d}\left\{ \ip{b}{\lambda}-\frac{1}{2\alpha}\norm{\lambda}_2^2-\log\left( \int_{\Omega}\exp\ip{C^T\lambda}{x}\dx{\mu(x)} \right) \right\}.
  \label{eq:Dual}
  \end{equation}
  Given a maximizer $\bar{\lambda}$ of \eqref{eq:Dual} one can recover a minimizer of \eqref{eq:Primal} via
  \begin{equation}
    \dx{\bar{\rho}}=\frac{\exp\ip{C^T\bar{\lambda}}{\cdot}}{\int_{\Omega}\exp\ip{C^T\bar{\lambda}}{\cdot}\dx{\mu}}\, \dx{\mu}.
    \label{eq:PrimalDual}
  \end{equation}
\label{main-thm}
\end{theorem}
\begin{proof}
  The dual problem can be obtained by applying the Fenchel-Rockafellar duality theorem (Theorem \ref{thm:FR}), with $f$ and $g$ defined in \eqref{eq:fg},
 to the primal problem
 \[
   \inf_{\rho\in\cM{\Omega}}\left\{ \mathcal{K}(\rho,\mu)+\frac{\alpha}{2}\norm{b-C\mathbb{E}_{\rho}[{X}]}_2^2 \right\},
 \]
 and substituting the expressions obtained in Lemmas \ref{lem:adjoint}, \ref{lem:KLconj} and \ref{lem:normconj}. All relevant conditions to apply this theorem have either been verified previously or are clearly satisfied.

 Note that $\bm{0}\subseteq\dom g^*=\R^d$ and $A^*\bm{0}=0\in C(\Omega)$, so
   \[
     A^*(\dom g^*)-\dom f^*\supseteq -\dom f^*=\left\{ \phi|-\phi\in\dom f^* \right\}=C(\Omega),
   \]
   since $\dom f^*=C(\Omega)$ by Lemma \ref{lem:KLconj}. Thus $0\in\interior\left(A^*\dom g^*-\dom f^*  \right)=C(\Omega)$, and Remark \ref{rem:primdual} is applicable. The primal-dual recovery formula \eqref{eq:PrimalDualRec} is given explicit form by Lemma \ref{lem:KLconj} by evaluating $\dx{\bar{\rho}_{\ip{C^T\bar{\lambda}}{\cdot}}}$. 

\end{proof}

The utility of the dual problem is that it permits a staggering dimensionality reduction, passing from an infinite-dimensional problem to a finite-dimensional one. 
Moreover, the form of the dual problem makes precise the role of $\alpha$ in \eqref{eq:Primal}. Notably in \cite[Cor. 4.9]{Borwein92} the problem
\begin{equation}\label{fid-par}
\inf_{\rho\in\cP{\Omega}\cap\dom\mathcal{K}(\cdot,\mu)}\mathcal{K}(\rho,\mu)\quad\text{s.t.} \norm{C\mathbb{E}_{\rho}[{X}]-b}_2^2\leq\frac{1}{2\alpha}
\end{equation} 
is paired in duality with \eqref{eq:Dual}. Thus the choice of $\alpha$ is directly related to the fidelity of $C\mathbb{E}_{\rho}[{X}]$ to the blurred image.
The following section is devoted to the choice of a prior and describing a method to directly compute $\mathbb{E}_{\bar{
\rho}}[{X}]$ from a solution of \eqref{eq:Dual}.

\subsection{Probabilistic Interpretation of Dual Problem}\label{sec-probint}
If no information is known about the original image, the prior $\mu$ is used to impose box constraints on the optimizer such that its expectation will be in the interval $[0,1]^d$ and will only assign non-zero probability to measurable subsets of this interval.
With this consideration in mind, the prior distribution should be the distribution of the random vector ${X}=[X_1,X_2,\dots]$ with the $X_i$ denoting independent random variables with uniform distributions on the interval $[u_i,v_i]$. If the $k^{th}$ pixel of the original image is unknown, we let $[u_k,v_k]=[0-\epsilon,1+\epsilon]$ for $\epsilon>0$ small in order to provide a buffer for numerical errors.

However, if the $k^{th}$ pixel of the ground truth image was known to have a value of $\ell$, one can enforce this constraint by taking the random variable $X_k$ to be distributed uniformly on $[\ell-\epsilon,\ell+\epsilon]$. Constructing $\mu$ in this fashion guarantees that its support (and hence $\Omega$) is compact. 

To deal with the integrals in \eqref{eq:Dual} and \eqref{eq:PrimalDual} it is convenient to note that (cf. \cite[Sec. 4.4]{Rohatgi01}) 
\[ \int_{\Omega}\exp\left(\ip{C^T\lambda}{x}\right)\dx{\mu}=\mathbb{M}_{{X}}[C^T\lambda],\] the moment-generating function of ${X}$ evaluated at $C^T\lambda$. Since the $X_i$ are independently distributed, $\mathbb{M}_{{X}}[t]=\Pi_{i=1}^d\mathbb{M}_{X_i}[t]$ \cite[Sec. 4.4]{Rohatgi01}, and since the $X_i$ are uniformly distributed on $[u_i,v_i]$ one has
 \[
   \mathbb{M}_{{X}}[t]=\prod_{i=1}^d\frac{e^{t_iv_i}-e^{t_iu_i}}{t_i(v_i-u_i)},
 \]
 and therefore the dual problem \eqref{eq:Dual} with this choice of prior can be written as
 \begin{equation}
   \max_{\lambda\in\R^d}\left\{ \ip{b}{\lambda}-\frac{1}{2\alpha}\norm{\lambda}^2_2-\sum_{i=1}^d\log\left( \frac{e^{C_i^T\lambda v_i}-e^{C_i^T\lambda u_i}}{C_i^T\lambda(v_i-u_i)} \right) \right\},
   \label{eq:SimplifiedDual}
 \end{equation} 
where $C_i$ denotes  the $i$-th column of $C$.
 A solution of \eqref{eq:SimplifiedDual} can be determined using a number of standard numerical solvers. We opted for the implementation \cite{Byrd95} of the L-BFGS algorithm due to its speed and efficiency. 

 Since only the expectation of the optimal probability measure for \eqref{eq:Primal} is of interest, we compute the $i^{th}$ component of the expectation $(\mathbb{E}_{\bar{\rho}}[ X])_i$ of the optimizer provided by the primal-dual recovery formula \eqref{eq:PrimalDual} via 
 \begin{align*}
   \frac{\int_{\Omega}x_i e^{\ip{C^T\bar{\lambda}}{x}}\dx{\mu}}{\int_{\Omega}e^{\ip{C^T\bar{\lambda}}{x}}\dx{\mu}}=\left.\partial_{t_i}\log\left( \int_\Omega e^{\ip{t}{x}}\, \dx{\mu} \right)\right\vert_{t=C^T\bar{\lambda}}.
   \end{align*}
Using the independence assumption on the prior, we obtain 
\[
  \mathbb{E}_{\bar{\rho}}[{X}]=\nabla_t \sum_{i=1}^d\log\left( \mathbb{M}_{X_i}[t] \right)
\]
such that the best estimate of the ground truth image is given by

 \begin{equation}
   \left(\mathbb{E}_{\bar{\rho}}[{X}]\right)_i=\frac{v_ie^{C_i^T\bar{\lambda}v_i}-u_ie^{C_i^T\bar{\lambda}u_i}}{e^{C_i^T\bar{\lambda}v_i}-e^{C_i^T\bar{\lambda}u_i}}-\frac{1}{C_i^T\bar{\lambda}}.
   \label{eq:SimplifiedPrimalDual}
 \end{equation}
 With \eqref{eq:SimplifiedDual} and \eqref{eq:SimplifiedPrimalDual} in hand, our entropic method for deconvolution can be implemented.

\subsection{Exploiting Symbology for PSF Calibration}\label{sec-MEM3}

 In order to implement blind deblurring on images that incorporate a symbology, one must first estimate the convolution kernel responsible for blurring the image. This step can be performed by analyzing the blurred symbolic constraints. We propose a method that is based on the entropic regularization framework discussed in the previous sections.

 In order to perform this kernel estimation step, we will use the same framework as  \eqref{eq:Primal} with $x$ taking the role of $c$. In the assumption that the kernel is of size $k\times k$, we take $\Omega=[0-\epsilon,1+\epsilon]^{k^2}$ for $\epsilon>0$ small (again to account for numerical error) and consider the problem
 \begin{equation}
   \inf_{\eta\in\cP{\Omega}}\left\{\mathcal{K}(\eta,\nu)+ \frac{\gamma}{2}\norm{\mathbb{E}_{\eta}[{X}]*\tilde{x}-\tilde{b}}^2_2\right\}.
   \label{eq:KerEst}
 \end{equation}
 Here, $\gamma>0$ is a parameter that enforces fidelity. $\tilde{x}$ and $\tilde{b}$ are the segments of the original and blurred image which are known to be fixed by the symbolic constraints. That is, $\tilde{x}$ consists solely of the embedded symbology and $\tilde{b}$ is the blurry symbology. By analogy with \eqref{eq:Primal}, the expectation of the optimizer of \eqref{eq:KerEst} is taken to be the estimated kernel. The role of $\nu\in\cP{\Omega}$ is to enforce the fact that the kernel should be normalized and non-negative (hence its components should be elements of $[0,1]$). Hence we take its distribution to be the product of $k^2$ uniform distributions on $[0-\epsilon,1+\epsilon]$. 
 As in the non-blind deblurring step, the expectation of the optimizer of \eqref{eq:KerEst} can be determined by passing to the dual problem (which is of the same form as \eqref{eq:SimplifiedDual}), solving the dual problem numerically and using the primal-dual recovery formula  \eqref{eq:SimplifiedPrimalDual}. A summary of the blind deblurring algorithm is compiled in Algorithm \ref{alg:blinddeblurring}. We would like to point out that the algorithm is not iterative, rather only one kernel estimate step and one deconvolution step are used.

 This method can be further refined to compare only the pixels of the symbology which are not convolved with pixels of the image which are unknown. By choosing these specific pixels, one can greatly improve the quality of the kernel estimate, as every pixel that was blurred to form the signal is known; however, this refinement limits the size of convolution kernel which can be estimated. 

\begin{algorithm}
  \floatname{Procedure}{Algorithm}
  \caption{ Entropic Blind Deblurring}
  \label{alg:blinddeblurring}
  \begin{algorithmic}
  
  \REQUIRE Blurred image $b$, prior $\mu$, kernel width $k$, fidelity parameters $\gamma,\alpha$;
  \ENSURE Deblurred image $\bar{x}$ 
  \STATE $\nu\gets$ density of $k^2$ uniformly distributed independent random variables
  \STATE $\lambda_{\bar{c}}\gets\argmax$ of analog of \eqref{eq:SimplifiedDual} for kernel estimate.
  \STATE $\bar{c} \gets$ expectation of $\argmin$ of \eqref{eq:KerEst} computed via analog of  \eqref{eq:SimplifiedPrimalDual} for kernel estimate evaluated at $\lambda_{\bar{c}}$
  \STATE $\lambda_{\bar{x}}\gets\argmax$ of \eqref{eq:SimplifiedDual}
  \STATE $\bar{x} \gets$ expectation of $\argmin$ of \eqref{eq:Primal} with kernel $\bar{c}$ computed via \eqref{eq:SimplifiedPrimalDual} evaluated at $\lambda_{\bar{x}}$\\
  \RETURN $\bar{x}$
\end{algorithmic}
\end{algorithm}

\section{Stability Analysis for Deconvolution}\label{sec-stability}

In contrast to, say, total variation methods, our maximum entropy method does not actively denoise. However, its 
ability to perform well with a denoising preprocessing step highlights that it is ``stable" to small perturbations in the data. 
In this section, we show that our convex analysis framework readily allows us to prove the following explicit stability estimate. 
\begin{theorem}\label{thm-stability}
  Let $b_1,b_2\in\R^d$, $C$ be non-singular, and  let 
  \[
    \rho_i\, = \, \underset{\rho\in\cP{\Omega}}{\arg\min} \left\{ \cK(\rho,\mu)+\frac{\alpha}{2}\norm{C\mathbb{E}_{\rho}[{X}]-b_i} \right\}\quad (i=1,2).\]
Then
\[
  \norm{\mathbb{E}_{\rho_1}[{X}]-\mathbb{E}_{\rho_2}[{X}]}_2\leq \frac{2}{\sigma_{\min}(C)}\norm{b_1-b_2}_2.
\]
Where $\sigma_{\min}(C)$ is the smallest singular value of $C$. 
\end{theorem}
We remark that the identical result holds true (with minor modifications to the proof) if the expectation is replaced with {\it any linear operator} $\cP{\Omega} \to \R^d$. 

The proof will follow from a sequence of lemmas. To this end we consider the optimal value function for \eqref{eq:Primal}, which we denote $v:\R^d\to\R$, as
\begin{equation}
  v(b) \,:= \, \inf_{\rho\in\cP{\Omega}}\left\{  \mathcal{K}(\rho,\mu) +\frac{\alpha}{2}\norm{C\mathbb{E}_{\rho}[{X}]-b}_2^2\right\}=\inf_{\rho\in \cP{\Omega}}\left\{k(\rho,b)+ h\circ L(\rho,b)\right\},
  \label{eq:OptVal}
\end{equation}
where 
\begin{equation}
  k:(\rho,b)\in\cM{\Omega}\times\R^d\mapsto\cK(\rho,\mu),\quad h=\frac{\alpha}{2}\norm{\cdot}_2^2,\quad L(\rho,b)=C\mathbb{E}_{\rho}[{X}]-b.  
  \label{eq:FG}
\end{equation}

The following results will allow us to conclude that $\nabla v$ is (globally) $\alpha$-Lipschitz. 

\begin{lemma}
  The operator $L$ in \eqref{eq:FG} is linear and continuous in the product topology, its adjoint is the map $z\mapsto (\ip{C^Tz}{\cdot},-z)\in C(\Omega)\times\R^d$. 
  \label{lem:adj2}
\end{lemma}
\begin{proof}
  Linearity and continuity of this operator from the linearity and weak$^*$ continuiy of the expectation operator (cf. Lemma \ref{lem:adjoint}).
The adjoint is obtained as in Lemma \ref{lem:adjoint},
  \[
    \ip{C\mathbb{E}_{\rho}[{X}]-b}{z}=(\ip{C^Tz}{\cdot},\rho)+\ip{b}{-z}.
  \]
\end{proof}
Next, we compute the conjugate of $k+h\circ L$.
\begin{lemma}
  The conjugate of $k+h\circ L$ defined in \eqref{eq:FG} is the function
\begin{equation}
  (\phi,y)\in C(\Omega)\times\R^d\mapsto(\cK(\cdot,\mu))^*(\phi+\ip{C^Ty}{\cdot})+\frac{1}{2\alpha}\norm{y}_2^2,
    \label{eq:FGconj}  
  \end{equation}
where $(\cK(\cdot,\mu))^*$ is the conjugate computed in Lemma \ref{lem:KLconj}. 
\end{lemma}
\begin{proof}
  Since $\dom h=\R^d$, $h$ is continuous and $k$ is proper, there exists $x\in L\dom k\cap\dom h$ such that $h$ is continuous at $x$. The previous condition guarantees that,  \cite[Thm. 2.8.3]{Zalinescu02} 
  \begin{equation}
    (k+h\circ L)^*(\phi,y)=\min_{z\in\R^d}\left\{ k^*( (\phi,y)-L^*(z))+h^*(z) \right\}.
    \label{eq:khl}
  \end{equation}
  The conjugate of $k$ is given by
  \[
    k^*(\phi,y)=\sup_{\substack{\rho\in\cM{\Omega}\\ b\in\R^d}}\left\{ (\phi,\rho)+\ip{y}{b}-\cK(\rho,\mu) \right\}.
  \]
  For $y\neq 0$, $\sup_{\R^d}\ip{y}{\cdot}=+\infty$. Thus,
  \[
    k^*(\phi,y)=\sup_{\rho\in\cM{\Omega}}\left\{(\phi,\rho)-\cK(\rho,\mu)  \right\}+\delta_{\{0\}}(y)=(\cK(\cdot,\rho) )^*(\phi)+\delta_{\{0\}}(y).
  \]
The conjugate of $h$ was established in Lemma \ref{lem:normconj} and the adjoint of $L$ is given in Lemma \ref{lem:adj2}. Substituting these expressions into \eqref{eq:khl} yields,
    \begin{align*}
      (k+h\circ L)^*(\phi,y)&=\min_{z\in\R^d}\left\{ (\cK(\cdot,\mu) )^*( (\phi-\ip{C^Tz}{\cdot})+\delta_{\{0 \}}(y+z)+\frac{\alpha}{2}\norm{z}_2^2 \right\}\\&=(\cK(\cdot,\mu) )^*(\phi+\ip{C^Ty}{\cdot})+\frac{1}{2\alpha}\norm{y}_2^2.
    \end{align*}

\end{proof}
The conjugate computed in the previous lemma can be used to establish that of the optimal value function. 
\begin{lemma}
  The conjugate of $v$ in \eqref{eq:OptVal} is
  $v^*:y\in\R^d\mapsto(\cK(\cdot,\mu))^*(\ip{C^Ty}{\cdot})+\frac{1}{2\alpha}\norm{y}_2^2$ which is $\frac{1}{\alpha}$-strongly convex.
  \label{lem:strongconv}
\end{lemma}
\begin{proof}
We begin by computing the conjugate,
\begin{align*}
  v^*(y)&=\sup_{b\in\R^d}\left\{\ip{y}{b}-\inf_{\rho\in\cM{\Omega}}\left\{ k(\rho,b)+h\circ L(\rho,b) \right\}  \right\}\\&=\sup_{\substack{\rho\in\cM{\Omega}\\b\in\R^d}}\left\{(0,\rho)+\ip{y}{b}- k(\rho,b)-h\circ L(\rho,b) \right\}\\&=(k+h\circ L)^*(0,y).
\end{align*}
In light of \eqref{eq:FGconj}, $v^*(y)=(\cK(\cdot,\mu))^*(\ip{C^Ty}{\cdot})+\frac{1}{2\alpha}\norm{y}_2^2$ which is the sum of a convex function and a $\frac{1}{\alpha}$-strongly convex function and is thus $\frac{1}{\alpha}$-strongly convex. 
\end{proof}

\begin{remark} 
  Theorem \ref{main-thm} establishes attainment for the problem defining $v$ in \eqref{eq:OptVal}, so $\dom v=\R^d$ and $v$ is proper. Moreover, \cite[Prop. 2.152]{Bonnans00} and \cite[Prop. 2.143]{Bonnans00} establish, respectively, continuity and convexity of $v$. Consequently, $(v^*)^*=v$ \cite[Thm. 2.3.3]{Zalinescu02} and since $v^*$ is $\frac{1}{\alpha}$-strongly convex, $v$ is G{\^a}teaux differentiable with globally $\alpha$-Lipschitz derivative \cite[Rmk. 3.5.3]{Zalinescu02}. 
  \label{rem:alphalip}
\end{remark}



We now compute the derivative of $v$.

\begin{lemma}
  The derivative of $v$ is the map
  $b\mapsto\alpha\left(b-C\mathbb{E}_{\bar{\rho}}[{X}]\right)$,
  where $\bar{\rho}$ is the solution of the primal problem \eqref{eq:Primal}, which is given in \eqref{eq:PrimalDual}.
\label{lem:Derivative}
\end{lemma}
\begin{proof}
 By \cite[Thm. 2.6.1]{Zalinescu02} and \cite[Thm. 2.8.3]{Zalinescu02},
  \[
    s\in\partial v(b)\iff (0,s)\in\partial \left(k+h\circ L\right)(\bar{\rho},b)=\partial k(\bar{\rho},b)+L^*(\partial h(L(\bar{\rho},b))),
  \]
  for $\bar{\rho}$ satisfying $v(b)=k(\bar{\rho},b)+h\circ L(\bar{\rho},b)$. Since $k$ is independent of $b$, $\partial k(\bar{\rho},b)=\partial\cK(\bar{\rho},\mu)\times\left\{ 0 \right\}$. By Lemma \ref{lem:adj2}, 
  \[
    L^*(\partial h (L(\bar{\rho},b)))=L^*(\alpha\left( C\mathbb{E}_{\bar{\rho}}[{X}]-b \right) )=\left( \ip{\alpha C^T(C\mathbb{E}_{\bar{\rho}}[{X}]-b)}{\cdot},\alpha(b-C\mathbb{E}_{\bar{\rho}}[{X}]) \right),
  \]
  so $s=\partial v(b)=\alpha(b-C\mathbb{E}_{\bar{\rho}}[X])$.
\end{proof}
We now prove Theorem \ref{thm-stability}.

\begin{proof}[Proof of Theorem \ref{thm-stability}.]
  By Lemma \ref{lem:strongconv}, $v^*$ is $\frac{1}{\alpha}$-strongly convex, so $\nabla v$ computed in Lemma \ref{lem:Derivative} is globally $\alpha$-Lipschitz (cf. Remark \ref{rem:alphalip}), thus 
  \[
    \norm{\nabla v(b_1)-\nabla v(b_2)}_2\leq \alpha\norm{b_1-b_2}_2
  \]
and
\begin{align*}
\norm{\nabla v(b_1)-\nabla v(b_2)}_{2}&=\alpha\norm{b_1-b_2+C(\mathbb{E}_{\rho_2}[{X}]-\mathbb{E}_{\rho_1}[{X}])}_2
  \\&\geq \alpha\norm{C\left( \mathbb{E}_{\rho_2}[{X}]-\mathbb{E}_{\rho_1}[{X}] \right)}_2-\alpha\norm{b_2-b_1}_2
  \\&\geq\alpha\sigma_{\min}(C)\norm{\mathbb{E}_{\rho_1}[{X}]-\mathbb{E}_{\rho_2}[{X}]}_2-\alpha\norm{b_1-b_2}_2.
\end{align*}
Consequently, $\norm{\mathbb{E}_{\rho_1}[{X}]-\mathbb{E}_{\rho_2}[{X}]}_2\leq\frac{2}{\sigma_{\min}(C)}\norm{b_1-b_2}_{2}$.
\end{proof}

%

\section{Numerical Results}\label{sec-results}

 \begin{figure}[!htb]
   \captionsetup[subfigure]{justification=centering}
   \begin{subfigure}{0.4\textwidth}
     \includegraphics[width=0.725\textwidth]{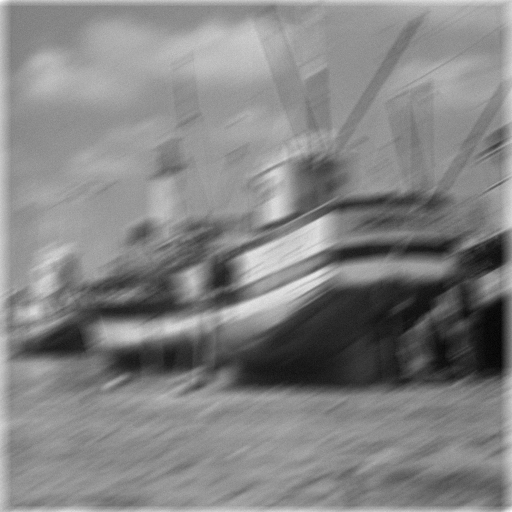} 
     \includegraphics[width=0.23\textwidth]{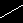}
     \caption{PSNR: 20.83 dB\\\phantom{PSNR}}
   \end{subfigure}
   \begin{subfigure}{0.29\textwidth}
     \centering
     \includegraphics[width=\textwidth]{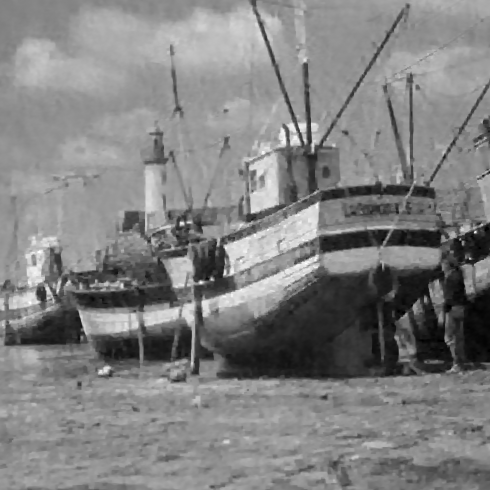}
    \caption{Cho \etal\cite{Cho11},\\ PSNR: 27.50 dB}
   \end{subfigure}
   \begin{subfigure}{0.29\textwidth}
     \centering
     \includegraphics[width=\textwidth]{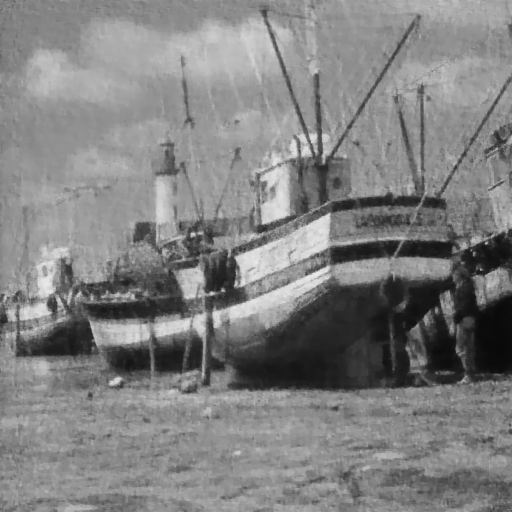}
    \caption{Ours,\\ PSNR: 26.68 dB}
  \end{subfigure}

   \caption{{\bf Deconvolution with noise: } Original image is 512$\times$512 pixels. (a) is the blurred image which is further degraded with $1\%$ Gaussian noise along with the $23$ pixel wide convolution kernel. 
(b)  is the result obtained using Cho {\it et al}'s  deconvolution method \cite{Cho11}.
(c) is the result obtained from the blurred image via our non-blind deblurring method. 
}
\label{fig:NoisyNonBlind}
 \end{figure}


\begin{figure*}[!htb]
  \begin{subfigure}{\textwidth}
   \begin{subfigure}{0.19\textwidth}
     \centering
   \phantom{\{}\\
     PSNR: 13.79 dB\\
   \includegraphics[width=\textwidth]{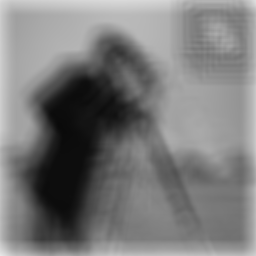}
   \\
\includegraphics[width=0.5\textwidth]{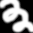}
  \end{subfigure}
   \begin{subfigure}{0.19\textwidth}
     \centering
     Liu \etal \cite{Liu19}\\
     PSNR: 16.39 dB\\
     \includegraphics[width=\textwidth]{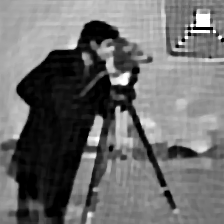}\\
      \includegraphics[width=0.5\textwidth]{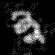}
   \end{subfigure}
  \begin{subfigure}{0.19\textwidth}
    \centering
    Pan \etal \cite{Pan16}\\ 
    PSNR: 16.49 dB\\
    \includegraphics[width=\textwidth]{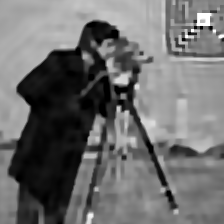}\\
\includegraphics[width=0.5\textwidth]{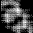}
  \end{subfigure}
    \begin{subfigure}{0.19\textwidth}
      \centering
      Yan \etal \cite{Yan17}\\
      PSNR: 16.54 dB\\
      \includegraphics[width=\textwidth]{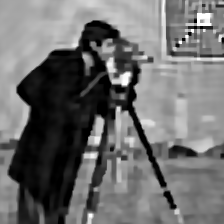} \\
 \includegraphics[width=0.5\textwidth]{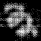}
    \end{subfigure}
   \begin{subfigure}{0.19\textwidth}
     \centering
     \phantom{\{}Ours\phantom{\{}\\
         PSNR: 29.44 dB\\
     \includegraphics[width=\textwidth]{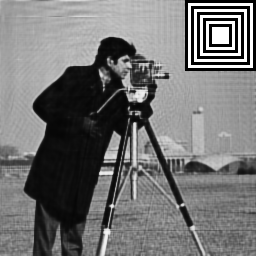}\\
  \includegraphics[width=0.5\textwidth]{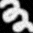}
   \end{subfigure} 
   \caption{}
 \end{subfigure}

  \begin{subfigure}{\textwidth}
   \begin{subfigure}{0.19\textwidth}
     \centering
    PSNR: 13.54 dB\\
   \includegraphics[width=\textwidth]{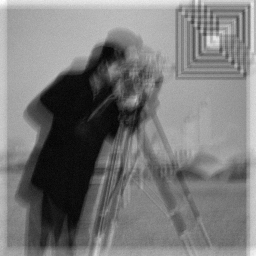}
   \\
\includegraphics[width=0.5\textwidth]{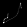}
  \end{subfigure}
   \begin{subfigure}{0.19\textwidth}
     \centering
     PSNR: 25.98 dB\\
     \includegraphics[width=\textwidth]{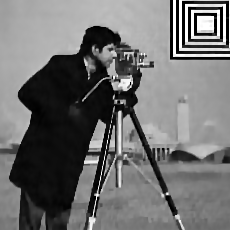}\\
      \includegraphics[width=0.5\textwidth]{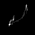}
   \end{subfigure}
  \begin{subfigure}{0.19\textwidth}
    \centering
    PSNR: 23.39 dB\\
    \includegraphics[width=\textwidth]{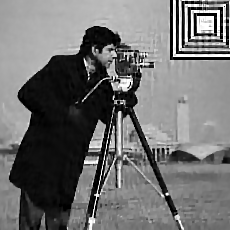}\\
\includegraphics[width=0.5\textwidth]{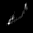}
  \end{subfigure}
    \begin{subfigure}{0.19\textwidth}
      \centering
    PSNR: 23.71 dB\\
      \includegraphics[width=\textwidth]{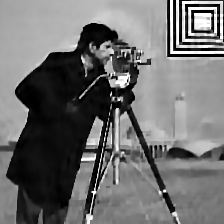} \\
 \includegraphics[width=0.5\textwidth]{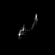}
    \end{subfigure}
   \begin{subfigure}{0.19\textwidth}
     \centering
    PSNR: 27.79 dB\\
     \includegraphics[width=\textwidth]{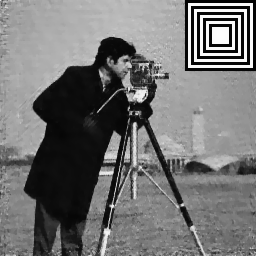}\\
  \includegraphics[width=0.5\textwidth]{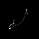}
   \end{subfigure} 
   \caption{}
 \end{subfigure}
  \begin{subfigure}{\textwidth}
   \begin{subfigure}{0.19\textwidth}
     \centering
     PSNR: 15.37 dB\\ 
   \includegraphics[width=\textwidth]{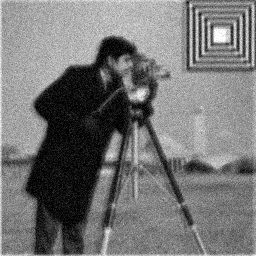}
   \\
\includegraphics[width=0.5\textwidth]{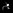}
  \end{subfigure}
   \begin{subfigure}{0.19\textwidth}
     \centering
     PSNR: 23.17 dB\\
     \includegraphics[width=\textwidth]{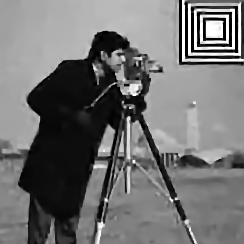}\\
      \includegraphics[width=0.5\textwidth]{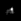}
   \end{subfigure}
  \begin{subfigure}{0.19\textwidth}
    \centering
    PSNR: 21.60 dB\\
    \includegraphics[width=\textwidth]{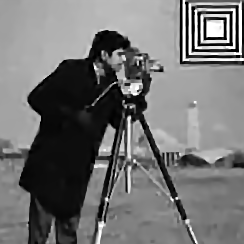}\\
\includegraphics[width=0.5\textwidth]{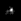}
  \end{subfigure}
    \begin{subfigure}{0.19\textwidth}
      \centering
      PSNR: 20.97 dB\\
      \includegraphics[width=\textwidth]{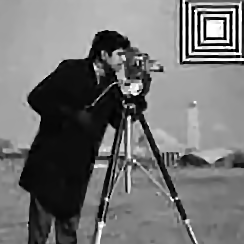} \\
 \includegraphics[width=0.5\textwidth]{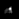}
    \end{subfigure}
   \begin{subfigure}{0.19\textwidth}
     \centering
     PSNR: 25.67 dB\\
     \includegraphics[width=\textwidth]{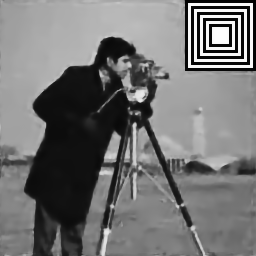}\\
  \includegraphics[width=0.5\textwidth]{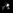}
   \end{subfigure} 
   \caption{}
 \end{subfigure}

 \caption{{\bf Blind deblurring with and without noise: } Original image is 256$\times$256 pixels. The performance, with varying amounts of noise and different blurring kernels, of our blind deblurring method with EPLL denoising preprocessing and TV denoising postprocessing to that of other contemporary methods with the EPLL denoising preprocessing step.   
    The blurred and noisy image is on the left with the original convolution kernel below it. 
    (a) is noiseless with a $33$ pixel wide kernel. (b) has $1\%$ Gaussian noise with a $27$ pixel wide kernel. 
    (c) has $5\%$ Gaussian Noise with a $13$ pixel wide kernel. We repeat the {\bf strong caveat} of this comparison: Unlike the other methods, ours directly exploits the known symbology. 
  }
 \label{fig:BlindCam}
 \end{figure*}


 \begin{figure}[!htb]
   \begin{subfigure}{0.49\textwidth}
     \centering
    \includegraphics[width=\textwidth]{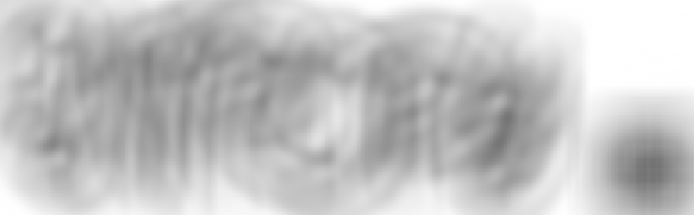}\\
    \includegraphics[width=0.3\textwidth]{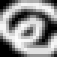}\hspace{3em}
    \includegraphics[width=0.3\textwidth]{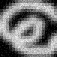}\\
    \includegraphics[width=\textwidth]{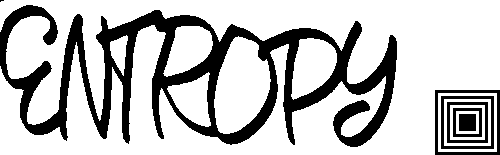}
    \caption{}
   \end{subfigure}
   \begin{subfigure}{0.49\textwidth}
     \centering
    \includegraphics[width=\textwidth]{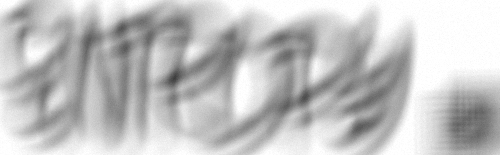}\\
    \includegraphics[width=0.3\textwidth]{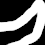}\hspace{3em}
    \includegraphics[width=0.3\textwidth]{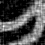}\\
    \includegraphics[width=\textwidth]{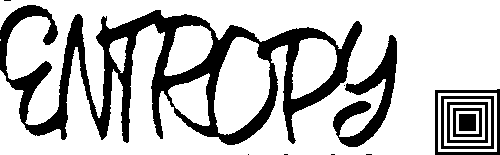}
    \caption{}
   \end{subfigure}
   \caption{{\bf Blind text deblurring with and without noise: }  Original image is 500$\times$155 pixels. Top: Blurred and noisy image. Middle: Original convolution kernel on the left and estimated kernel on the right. Bottom: Deblurred image obtained using our method with an EPLL denoising preprocessing step and a thresholding postprocessing step. (a) is noiseless with a $57$ pixel kernel. (b) has $1\%$ Gaussian noise with a $45$ pixel wide kernel.
}
\label{fig:BlindText}
 \end{figure}

 We present results obtained using our method on certain simulated images.
 We begin with deconvolution, i.e. when the blurring kernel $c$ is known. Figure \ref{fig:NoisyNonBlind} provides an example in which a blurry and noisy image has been deblurred using the non-blind deblurring method. We note that the method does not actively denoise blurred images when a uniform prior is used, so a preprocessing step consisting of expected patch log likelihood (EPLL) denoising \cite{Zoran11} is first performed. For the sake of consistency, the same preprocessing step is applied prior to using Cho {\it et al}'s deconvolution method \cite{Cho11} (this step also improves the quality of the restoration for this method). The resulting image is subsequently deblurred and finally TV denoising \cite{ROF} is used to smooth the image in our case (this step is unnecessary for the other method as it already results in a smooth restoration).  Note that for binary images such as text, TV denoising can be replaced by a thresholding step (see figure \ref{fig:BlindText}).



 \begin{figure}[!htb]
   \begin{subfigure}{0.19\textwidth}
     \centering
   \phantom{\{}\\
     PSNR: 14.61 dB\\
   \includegraphics[width=\textwidth]{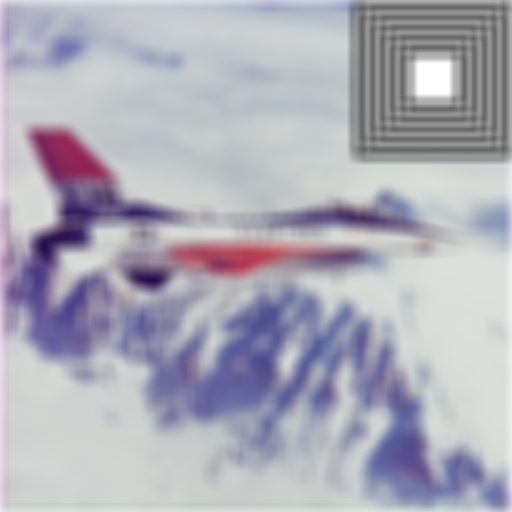}
   \\
\includegraphics[width=0.5\textwidth]{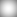}
\caption{}
  \end{subfigure}
   \begin{subfigure}{0.19\textwidth}
     \centering
     Liu \etal \cite{Liu19}\\
     PSNR: 26.87 dB\\
     \includegraphics[width=\textwidth]{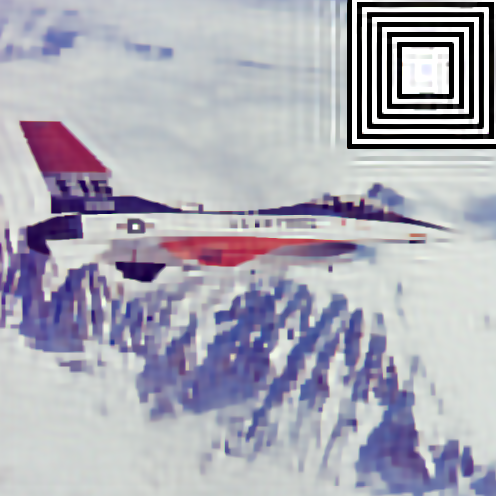}\\
     \includegraphics[width=0.5\textwidth]{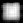}
\caption{}
   \end{subfigure}
  \begin{subfigure}{0.19\textwidth}
    \centering
    Pan \etal \cite{Pan16}\\ 
    PSNR: 23.97 dB\\
    \includegraphics[width=\textwidth]{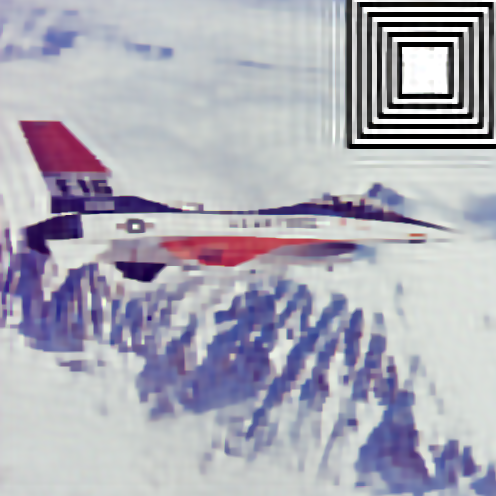}\\
\includegraphics[width=0.5\textwidth]{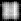}
\caption{}
\end{subfigure}
    \begin{subfigure}{0.19\textwidth}
      \centering
      Yan \etal \cite{Yan17}\\
      PSNR: 24.67 dB\\
      \includegraphics[width=\textwidth]{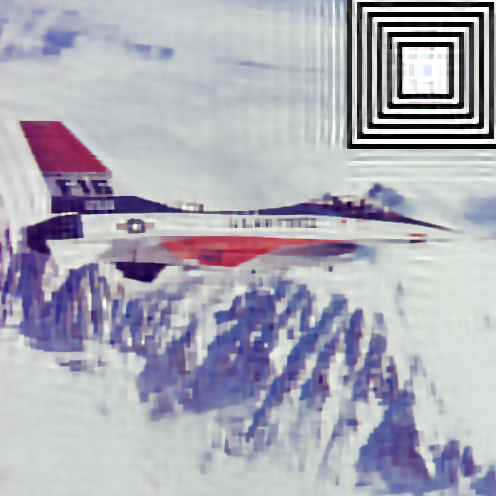} \\
 \includegraphics[width=0.5\textwidth]{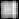}
    \caption{}
    \end{subfigure}
   \begin{subfigure}{0.19\textwidth}
     \centering
     \phantom{\{}Ours\phantom{\{}\\
         PSNR: 39.66 dB\\
     \includegraphics[width=\textwidth]{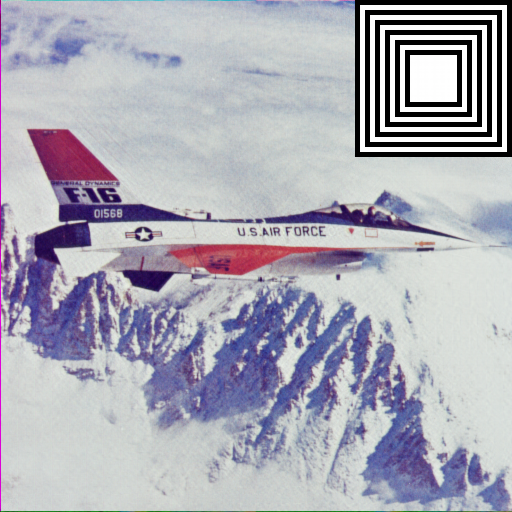}\\
  \includegraphics[width=0.5\textwidth]{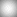}
 \caption{}  
\end{subfigure} 
\caption{{\bf Blind deblurring in color: }  Original image is 512$\times$512 pixels. (a) is the image which has been blurred with a $17\times 17$ kernel (no noise). (b)-(e) are the latent image and estimated kernel obtained with different methods. We repeat the {\bf strong caveat} of this comparison: Unlike the other methods, ours directly exploits the known symbology.}   
  \label{fig:BlindPlane} \end{figure}

  Results for blind deblurring are compiled in Figure \ref{fig:BlindCam}, \ref{fig:BlindText} and \ref{fig:BlindPlane}. In this case $\gamma=10^{5}$ and $\alpha=10^{6}$ provide good results in the noiseless case and $\gamma=10^{3},\alpha=10^{4}$ is adequate for the noisy case, these parameters require manual tuning to yield the best results however. Comparisons are provided with various state of the art methods \cite{Liu19,Pan16,Yan17}. These methods estimate the kernel and subsequently use known deconvolution algorithms to generate the latent image, the same deconvolution method \cite{Cho11} was used for all three methods as it yielded the best results.

We have stressed that the other methods in our comparison do not exploit the presence of a finder pattern; they are fully blind where ours are symbology-based. Hence it would be natural to ask if these methods can also benefit from known symbology. It is not immediate how to 
successfully use these methods with 
symbology. 
Most of these methods  iterate between optimizing for the image and the kernel, and  employ $\ell_0$ regularization which is non-convex and can therefore terminate in a local minimum. Hence iteration could prove problematic. Following what we have done here, one could use  a strictly convex  regularization scheme  to first estimate $c$ using the finder pattern as the sole image and then deconvolute. But this was precisely the approach taken by Gennip  et al in  \cite{Gennip15}  using a strictly convex  regularization scheme of the form (8) to exploit the known finder patterns in QR barcodes. 
Its performance was significantly inferior to our MEM method as presented in \cite{Riouxetal}. 
The ability of MEM to incorporate nonlinear constraints via the introduction of a prior is a definite advantage. 

In  \ref{appendix-D},  we compensate for the bias (in our favour) in the comparisons of our symbology-based blind deblurring method with fully blind methods by 
  presenting a comparison which clearly  gives  the favorable  bias to the other method. 
We consider the same examples as in figures \ref{fig:BlindCam} and \ref{fig:BlindPlane} but compare our 
 symbology-based blind method with the deconvolution (non-blind) method of Cho et al. \cite{Cho11}; that is, we give the comparison method the advantage of knowing the PSF.

\subsection{The Effects of Noise}

In the presence of additive noise, attempting to deblur images using methods that are not tailored for noise is generally ineffective. Indeed, the image acquisition model $b=c*x$ is replaced by $b=c*x+n$ where $n$ denotes the added noise. The noiseless model posits that the captured image should be relatively smooth due to the convolution, whereas the added noise sharpens segments of the image randomly, so the two models are incompatible. However, Figures \ref{fig:NoisyNonBlind} and \ref{fig:BlindCam} show that our method yields good results in both deconvolution and blind deblurring when a denoising preprocessing step (the other methods use the preprocessed version of the image as well for the sake of consistency) and a smoothing postprocessing step are utilized.

 Remarkably, with a uniform prior, the blind deblurring method is more robust to the presence of additive noise in the blurred image than the non-blind method. Indeed, accurate results were obtained with up to $5\%$ Gaussian noise in the blind case whereas in the non-blind case, quality of the recovery diminished past $1\%$ Gaussian noise. This is due to the fact that the preprocessing step fundamentally changes the blurring kernel of the image. We are therefore attempting to deconvolve the image with the wrong kernel, thus leading to aberrations. On the other hand, the estimated kernel for blind deblurring is likely to approximate the kernel modified by the preprocessing step, leading to better results. Moreover, a sparse (Poisson) prior was used in the kernel estimate for the results in Figure \ref{fig:BlindCam} so as to mitigate the effects of noise on the symbology. 

 Finally, we note that there is a tradeoff between the magnitude of blurring and the magnitude of noise. Indeed, large amounts of noise can be dealt with only if the blurring kernel is relatively small and for large blurring kernels, only small amounts of noise can be considered. This is due to the fact that for larger kernels, deviations in kernel estimation affect the convolved image to a greater extent than for small kernels.  

\section{The Role of the Prior, Denoising and Further Extensions}\label{sec-pror}

Our method is based upon the premise that a priori the probability density $\rho$ at each pixel is independent from the other pixels. Hence in our model, the only way to introduce correlations between pixels is via the prior $\mu$. 
Let us first recall the role of the prior $\mu$ in the deconvolution (and $\nu$ in the kernel estimation). 
In deconvolution for general images, the prior $\mu$ was only used to impose box constraints; otherwise, it was unbiased (uniform). 
For deconvolution with symbology, e.g. the presence of a known finder pattern, this information was directly imposed on the prior. For kernel estimation, the prior $\nu$ was used to enforce normalization and positivity of the kernel; but otherwise unbiased. 

Our general method, on the other hand, facilitates the incorporation of far more prior information. 
Indeed,  we seek a prior probability distribution $\mu$
over the space of latent images that possesses at least one of the
following two properties:
\begin{enumerate}
\item $\mu$ has a tractable moment-generating function (so that the dual
problem can be solved via gradient-based methods such as L-BFGS), 
\item It is possible to efficiently sample from $\mu$ (so that the dual
problem can be solved via stochastic optimization methods).
\end{enumerate}

As a simple example, we provide a comparison between a uniform and an exponential prior with large rate parameter ($\beta=400$ at every pixel) to deblur a text image corrupted by $5\%$ Gaussian noise with no preprocessing or postprocessing in figure \ref{fig:Priors}. In the former case, we set the fidelity parameter $\alpha=3\times 10^{4}$ and in the latter, $\alpha=10^{4}$. It is clear from this figure that the noise in the blurred image is better handled by the exponential prior. This fact will be further discussed in Section \ref{sec-reform} which also introduces an efficient implementation of the MEM with an exponential prior. 
In this case, sparsity has been used to promote the presence of a white background by inverting the intensity of the channels during the deblurring process. 
 
\begin{figure}[!htb]
   \captionsetup[subfigure]{justification=centering}
     
   \begin{subfigure}{0.4\textwidth}
     \includegraphics[width=0.725\textwidth]{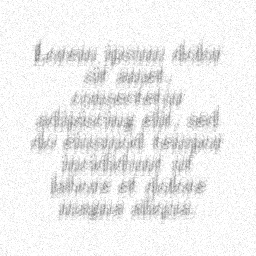} 
    \includegraphics[width=0.23\textwidth]{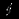}
  \caption{PSNR: 15.69 dB\\\phantom{PSNR}}   \end{subfigure}
   \begin{subfigure}{0.29\textwidth}
     \centering
     \includegraphics[width=\textwidth]{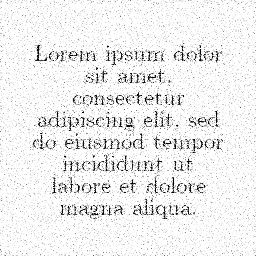} 
     \caption{Uniform,\\PSNR: 19.18 dB}
   \end{subfigure}
   \begin{subfigure}{0.29\textwidth}
     \centering
     \includegraphics[width=\textwidth]{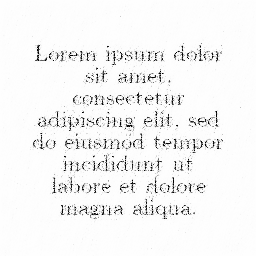}
     
     \caption{Exponential,\\PSNR: 20.73 dB}
   \end{subfigure}
   \caption{{\bf Deconvolution with different priors: }  Original image is 256$\times$256 pixels. (a) is the blurred image with added $5\%$ Gaussian noise along with the $19$ pixel wide convolution kernel.  
(b) is the result obtained using a uniform prior. 
(c) is obtained using an exponential prior. 
}
\label{fig:Priors}
 \end{figure}

More generally, we believe our method could be tailored to contemporary approaches for priors used in machine learning, and this could be one way of blind deblurring without the presence of a finder pattern. 
A natural candidate for such a prior $\mu$ is a {\bf generative adversarial
network (GAN)} (cf. \cite{Good}) trained on a set of instances from a class of natural
images (such as face images). GANs have achieved state-of-the-art
performance in the generative modelling of natural images (cf. \cite{8953766})
and it is possible, by design, to efficiently
sample from the distribution implicitly defined by a GAN's generator.
Consequently, when equipped with a pre-trained GAN prior, our dual
problem (12) would be tractable via stochastic compositional optimization
methods such as the ASC-PG algorithm of Wang et al. in \cite{Wang}.

\subsection{A Return to the Classical Formulation at the Image Level}\label{sec-reform}

It can be advantageous, for example in directly relating the role of the prior to the image regularization, 
 to reformulate our MEM primal problem  \eqref{eq:Primal} at the image  level. 
Recall that previous MEM approaches for inverse problems \eqref{oldMEM-1}, what we called the classical approach, were all based on a primal problem on the space of images. 
Our formulation can also be rephrased at the image level as follows: find $\bar{x}$,  the estimate of the ground truth image, where 
\begin{equation}\label{reform}
 \underset{x\in\R^d}{\arg\min} 
\left\{ v(x)+\frac{\alpha}{2}\norm{Cx-b}_2^2 \right\}\qquad\text{with }\,\, v(x)\,:=\, \inf_{\rho} \left\{ \cK(\rho,\mu) \, \big\vert \, \mathbb{E}_{\rho}[X]=x \right\}.
\end{equation}
In this problem, which essentially appears in
\cite{marechal1998principle, Besnerais99},
one can think of $v(x)$ as a {\it regularizer} for the image estimate~$x$.

Given the structure of the above problem as the sum of a potentially, lower semicontinuous convex function and a smooth convex function with $L$-Lipschitz gradient, the Fast Iterative Shrinkage-Thresholding Algorithm (FISTA) \cite{Beck09} can be utilized provided the proximal operator, defined for $t>0$ by 
\[
  \prox_{tv}(u)\, := \,  \underset{x\in\R^d}{\arg\min}  \left( v(x)+\frac{1}{2t}\norm{x-u}_2^2 \right),\]
can be computed efficiently. As before, one can consider a dual formulation to the problem defining the proximal operator (see \eqref{oldMEM-2} for the conjugate of $v$)
\begin{equation}
  \max_{\lambda\in\R^d}\left\{ \ip{u}{\lambda}-\frac{t}{2}\norm{\lambda}_2^2-\log\left(\mathbb{M}_{\mu}[\lambda] \right) \right\},
  \label{eq:dual3}
\end{equation}
such that 
\[
  \prox_{tv}(u)=\left.\nabla_t\log\left( \mathbb{M}_{\mu}[t] \right)\right\vert_{\bar{\lambda}},
\]
for $\bar{\lambda}$ a solution to \eqref{eq:dual3}.
Note that the Lipschitz constant of the derivative of $\frac{\alpha}{2}\norm{Cx-b}_2^2$ (which dictates the step size used in the FISTA iterations) is $\alpha\sigma_{\max}(C)$. Even if the largest singular value of $C$ is 
unknown, one can determine the step size using a line search.

One example for which the proximal operator can be computed efficiently is when the prior consists of independent exponential distributions at each pixel with respective rate parameters $\beta_i>0$. Indeed, \eqref{eq:dual3} reads in this case
\[
  \max_{\lambda\in\R^d}\left\{ \ip{u}{\lambda}-\frac{t}{2}\norm{\lambda}_2^2+\sum_{i=1}^d\log\left(1-\frac{\lambda_i}{\beta_i} \right) \right\},  
\]
whose solution can be written componentwise as 
\[
  (\bar{\lambda}_{\pm})_i=\frac{u_i+\beta_it\pm\sqrt{(u_i+\beta_it)^2-4t(u_i\beta_i-1)}}{2t}=\frac{u_i+\beta_it\pm\sqrt{(u_i-\beta_it)^2+4t}}{2t},
\]
thus one takes the smaller root ($\bar{\lambda}_{-}$) since the $\log$ moment-generating function is well-defined for $\frac{\lambda_i}{\beta_i}<1$, thus
\[
  \left(\prox_{tv}(u)\right)_i=\frac{1}{\beta_i-(\bar{\lambda}_{-})_i}.
\]
As such, one can implement the FISTA algorithm to perform deblurring via the MEM with an exponential prior. This method was used to generate the example in figure \ref{fig:Priors}. 

It is natural to seek a correspondence between regularization at the level of the probability measure using a fixed prior and the regularization at the image level (i.e. the reformulation). In the case of an exponential distribution, one has the following expression for the image space regularization \cite[Table 1]{Besnerais99}
\[
  v(x)=\sum_{i=1}^d x_i\beta_i-1-\log(x_i\beta_i),\quad (x_i>0).
\]
Note that if $x_i\beta_i$ is large, the contribution of that summand is dominated by the linear term. As such, taking an exponential prior with large rate parameter yields results in an image space regularization which approximates $\ell_1$ regularization.  

This subsection is simply meant to highlight this approach: A full study (theory and applications) of the formulation (\ref{reform})  is in progress. 

\subsection{Further extension}

We have used throughout the duality pairing between $\mathcal{M}(\Omega)$ and $\mathcal{C}(\Omega)$ with $\Omega \subset \R^d$ compact.
Notice that, since the Kullback-Leibler entropy takes finite values only for measures that
are absolutely continuous with respect to the reference measure, it is also possible to work
with the Radon-Nikodym derivatives, as in~\cite{marechal1998principle}. The primal problem
is then expressed in a space of measurable functions. This setting also facilitates an
interesting extension to the case where $\Omega$ is not bounded. As a matter of fact,
in the latter case, first-order moment integrals such that
$$
(\pi_k,\rho)=\int_\Omega x_k\,\mathrm{d}\rho(x)
$$
are not necessarily well-defined since the coordinate functions $\pi_k:x\mapsto x_k$ may be unbounded
on~$\Omega$. As shown in~\cite{marechal1998principle}, partially finite convex programming
can still be carried out in this setting, offering interesting possible extensions to our analysis.
In essence, in case $\Omega$ is unbounded one must restrict the primal problem to spaces of functions defined by an integrability
condition against a family of constraint functions. Such spaces are sometimes referred to as K\"othe
spaces, and their nature was shown to allow for the application of the convex dual machinery for
entropy optimization~\cite{marechal1998principle}. The corresponding extensions are currently
under consideration, and will give rise to interesting future work.

\section{Conclusion} 
 
The MEM method for the regularization of ill-posed deconvolution problems  garnered much attention in the 80's and 90's with imaging  applications in astrophysics and crystallography.  
However, it is surprising that since that time it has rarely been used for image deblurring (both blind and non blind), and  is not well-known in the image processing and machine learning communities. 
We have shown that a reformulation of the MEM principle produces an 
efficient (comparable with the state of the art) scheme for both deconvolution and kernel estimation for general images. It is also amenable to large blurs which are seldom used for testing methods. 
The scheme reduces to an unconstrained, smooth and strongly convex optimization problem in finite dimensions  for which there exist an abundance of black-box solvers. The strength of this higher-level method lies in its ability to incorporate prior information, often in the form of  nonlinear constraints. 
 
 For kernel estimation (blind deblurring), we focused our attention on exploiting a priori assumed symbology (a finder pattern). While this situation/assumption is indeed restrictive: (i)  there are  scenarios and applications, in addition to synthetic images like barcodes; 
(ii) It is far from clear how standard regularization based methods of the form (\ref{eq:stdapp}) can  be used to  exploit symbology to obtain a similar accuracy of kernel estimation.

 In general, the MEM method is stable with respect to small amounts of noise and this allowed us to successfully deblur noisy data by first pre conditioning with a state of the art denoiser. However, as shown in Section \ref{sec-pror},  the MEM method itself can be used for denoising with a particular choice of prior.

Finally, let us reiterate that in our numerical experiments we use only a modest amount of the potential of MEM to exploit prior information. Future work will concern both kernel estimation without the presence of finder patterns as well as a full study of effects of regularization via the image formulation discussed in Section \ref{sec-reform}.

 \ack{G.R. was partially supported by the NSERC CGS-M program, R.C. and T.H. were partially supported by the NSERC Discovery Grants program. We thank Yakov Vaisbourd for proposing the use of FISTA  in Section \ref{sec-reform}. 
 We would also like to thank the anonymous referees for their many comments which significantly improved the paper.}

\appendix

\section{Implementation Details}
   All figures were generated by implementing the methods in the Python programming language using the Jupyter notebook environment. 
 Images were blurred synthetically using motion blur kernels taken from \cite{Levin09} as well as Gaussian blur kernels to simulate out of focus blur. 
 The relevant convolutions are performed using fast Fourier transforms. 
 Images that are not standard test bank images were generated using the GNU Image Manipulation Program (GIMP), moreover this software was used to add symbolic constraints to images that did not originally incorporate them. 
 All testing was performed on a laptop with an Intel i5-4200U processor.
 The running time of this method depends on a number of factors such as the size of the image being deblurred, whether the image is monochrome or colour, the desired quality of the reproduction desired (controlled by the parameter $\alpha$) as well as  the size of the kernel and whether or not it is given. 
 If a very accurate result is required, these runtimes vary from a few seconds for a small monochrome text image blurred with a small sized kernel to upwards of an hour for a highly blurred colour image.

  \section{The pre and post-processing steps}
  We provide an example of the intermittent images generated in the process of deblurring a noisy image via our method. 
  \begin{figure}[!htb]
    \captionsetup[subfigure]{justification=centering}
    \begin{subfigure}{0.24\textwidth}
      \centering
      \includegraphics[width=\textwidth]{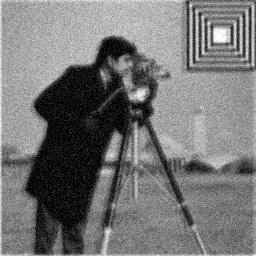}
      \caption{}
    \end{subfigure}
    \begin{subfigure}{0.24\textwidth}
      \centering
      \includegraphics[width=\textwidth]{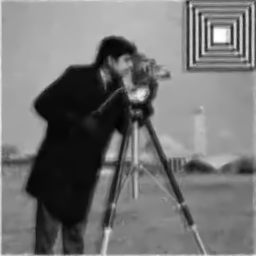}
      \caption{}
    \end{subfigure}
    \begin{subfigure}{0.24\textwidth}
      \centering
      \includegraphics[width=\textwidth]{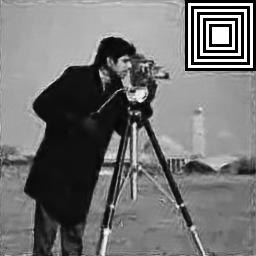}
      \caption{}
    \end{subfigure}
    \begin{subfigure}{0.24\textwidth}
      \centering
      \includegraphics[width=\textwidth]{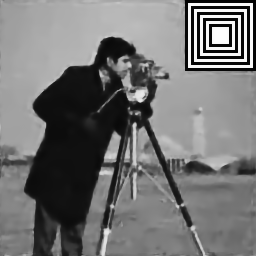}
      \caption{}
    \end{subfigure}
    \caption{(a) is the blurred and noisy image from Figure \ref{fig:BlindCam} (c). 
    (b) is the denoised image obtained via the method of Zoran and Weiss \cite{Zoran11}. (c) is the deblurred image obtained using our method. (d) is a smoothed version of (c) obtained via TV denoising \cite{ROF} using the implementation of Chambolle \cite{Chambolle04}. }
    \label{fig:prepost}
  \end{figure}

  \section{Parameters for Comparisons}
  We compile the parameters used for the kernel estimation step of the deblurring methods to which we compared our method. 

  \begin{figure}[!htb]
  \begin{tabular}{cccc}
    &Liu {\it et al} \cite{Liu19}&Pan {\it et al} \cite{Pan16}& Yan {\it et al} \cite{Yan17}\\
    Fig. \ref{fig:BlindCam} (a)
    &
    $\begin{array}{|ll|}
    \cline{1-2}
  \text{Kernel size}&55
\\  \shortstack{Kernel prior\\ parameter}&0.8
\\ \cline{1-2}
\end{array}$
&
    $\begin{array}{|ll|}
    \cline{1-2}
  \text{Kernel size}&37
\\\shortstack{Kernel prior\\ parameter}&1
\\ \cline{1-2}
\end{array}$
&
    $\begin{array}{|ll|}
    \cline{1-2}
  \text{Kernel size}&41
\\\shortstack{Kernel prior\\ parameter}&0.8
\\ \cline{1-2}
\end{array}$
\\


\\
    Fig. \ref{fig:BlindCam} (b)&
    $\begin{array}{|ll|}
    \cline{1-2}
  \text{Kernel size}&29
\\\shortstack{Kernel prior\\ parameter}&0.8
\\ \cline{1-2}
\end{array}$
    &
    $\begin{array}{|ll|}
    \cline{1-2}
  \text{Kernel size}&33
\\\shortstack{Kernel prior\\ parameter}&0.8
\\ \cline{1-2}
\end{array}$
    &
    $\begin{array}{|ll|}
    \cline{1-2}
  \text{Kernel size}&55
\\\shortstack{Kernel prior\\ parameter}&1
\\ \cline{1-2}
\end{array}$
    \\

    \\
    Fig. \ref{fig:BlindCam} (c)
    &
        $\begin{array}{|ll|}
    \cline{1-2}
  \text{Kernel size}&21
\\\shortstack{Kernel prior\\ parameter}&0.8
\\ \cline{1-2}
\end{array}$
&
    $\begin{array}{|ll|}
    \cline{1-2}
  \text{Kernel size}&27
\\\shortstack{Kernel prior\\ parameter}&0.8
\\ \cline{1-2}
\end{array}$
&
    $\begin{array}{|ll|}
    \cline{1-2}
  \text{Kernel size}&23
\\\shortstack{Kernel prior\\ parameter}&1
\\ \cline{1-2}
\end{array}$
    \\
    \\
    Fig. \ref{fig:BlindPlane} &
    $\begin{array}{|ll|}
    \cline{1-2}
  \text{Kernel size}&23
\\\shortstack{Kernel prior\\ parameter}&0.8
\\ \cline{1-2}
\end{array}$
    &
    $\begin{array}{|ll|}
    \cline{1-2}
  \text{Kernel size}&23
\\\shortstack{Kernel prior\\ parameter}&1
\\ \cline{1-2}
\end{array}$
    &

    $\begin{array}{|ll|}
    \cline{1-2}
  \text{Kernel size}&25
\\\shortstack{Kernel prior\\ parameter}&0.8
\\ \cline{1-2}
\end{array}$
  \end{tabular}
  \caption{This table compiles the parameter values used to estimate the kernels for the various methods. For all methods, the $\ell_0$ gradient parameter is set to $4e^{-3}$, the parameters for the surface-aware prior, dark channel prior and the extreme channel prior are all set to $4e^{-3}$ in their respective methods.  }
\end{figure}

Once the kernel has been estimated, the deconvolution method for images with outliers \cite{Cho11} was used to obtain the latent image. For this method, we set the standard deviation for inlier noise to $\frac{5}{255}$ and set the regularization strength for the sparse priors to $0.003$ and decrease it iteratively (with the same kernel) until we obtain a balance between the sharpness of the image and the amount of noise.

 \section{Comparison of our symbology-based Blind Method to a Non-Blind Deblurring Method}\label{appendix-D}
  Here we consider the same examples as in figures \ref{fig:BlindCam} and \ref{fig:BlindPlane} but compare our 
 symbology-based blind method with a state of art method deconvolution method of Cho et al. \cite{Cho11}; that is, we give the comparison method the advantage of knowing the PSF. 
 PSNR values for Cho's method were computed with a cropped version of the latent image to reduce the effects of the boundary conditions for the convolution. The choice of boundary condition accounts for some of our the higher PSNR values.  In images with noise, the non-blind deconvolution method was applied to both the noisy image and the denoised image (via our pre-denoising step), the better result is presented in the  figure D1.

 \begin{figure}[!htb]
    \captionsetup[subfigure]{justification=centering}
    \centering
    \begin{subfigure}{0.20\textwidth}
      \centering
      \includegraphics[width=\textwidth]{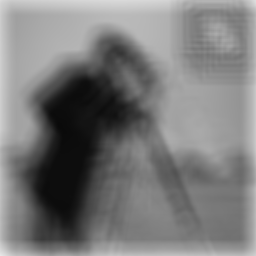} 
      \caption*{Input\\PSNR: 13.79 dB}   
    \end{subfigure}
    \begin{subfigure}{0.20\textwidth}
      \centering
    
      \includegraphics[width=\textwidth]{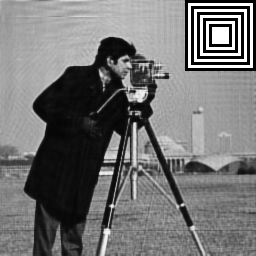}
    \caption*{Ours\\PSNR: 29.44 dB}   
    \end{subfigure}
    \begin{subfigure}{0.20\textwidth}
      \centering 
      \includegraphics[width=\textwidth]{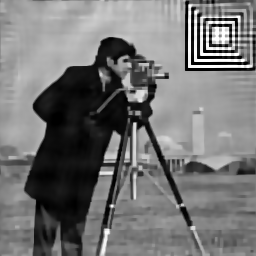}
      \caption*{Cho et al.\\PSNR: 23.73 dB}
    \end{subfigure}
\\
    \begin{subfigure}{0.20\textwidth}
      \centering
      \includegraphics[width=\textwidth]{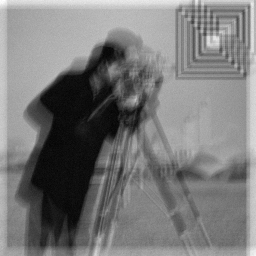}
      \caption*{Input\\PSNR: 13.54 dB}   
    \end{subfigure}
    \begin{subfigure}{0.20\textwidth}
      \centering
    
      \includegraphics[width=\textwidth]{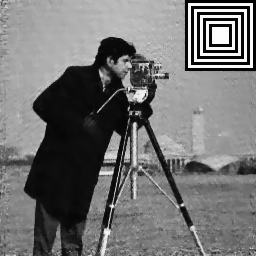}
      \caption*{Ours\\PSNR: 27.79 dB}   
    \end{subfigure}
    \begin{subfigure}{0.20\textwidth}
      \centering 
      \includegraphics[width=\textwidth]{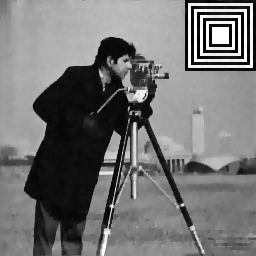}
    \caption*{Cho et al.\\PSNR: 28.50 dB}
    \end{subfigure}
  \\
    \begin{subfigure}{0.20\textwidth}
      \centering
      \includegraphics[width=\textwidth]{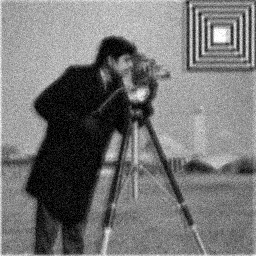}
    \caption*{Input\\PSNR: 15.37 dB}   
    \end{subfigure}
    \begin{subfigure}{0.20\textwidth}
      \centering
    
      \includegraphics[width=\textwidth]{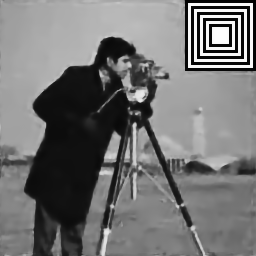}
    \caption*{Ours\\PSNR: 25.67 dB} 
    \end{subfigure}
    \begin{subfigure}{0.20\textwidth}
      \centering
    
      \includegraphics[width=\textwidth]{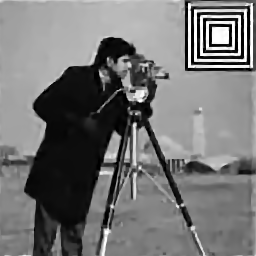}
      \caption*{Cho et al.\\PSNR: 25.12 dB}
    \end{subfigure}
    \\
    \begin{subfigure}{0.20\textwidth}
      \centering
        \includegraphics[width=\textwidth]{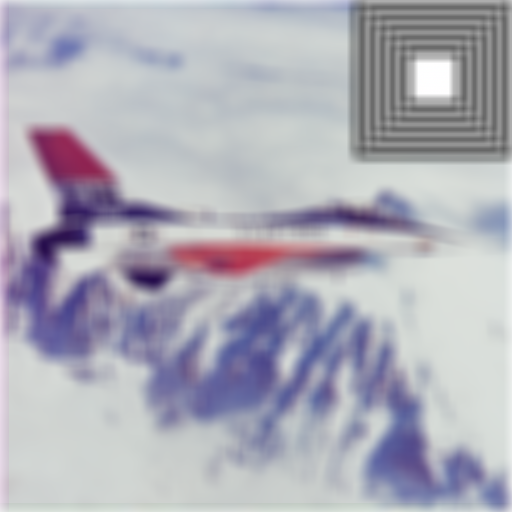}
        \caption*{Input\\PSNR: 14.61 dB} 
      \end{subfigure}
    \begin{subfigure}{0.20\textwidth}
      \centering
    
        \includegraphics[width=\textwidth]{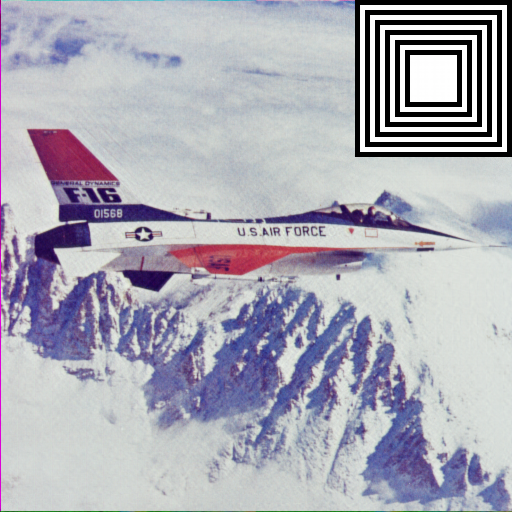}
        \caption*{Ours\\PSNR: 39.66 dB}
      \end{subfigure}
    \begin{subfigure}{0.20\textwidth}
      \centering
    
        \includegraphics[width=\textwidth]{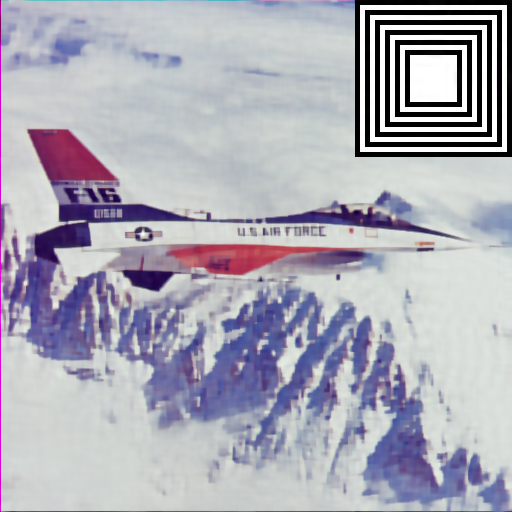}
      \caption*{Cho et al.\\PSNR: 31.52 dB}
    \end{subfigure}
    \label{fig-bias}
    \caption{}
  \end{figure}

 \section*{Bibliography}
\bibliography{MEMM}
\bibliographystyle{plain}

\end{document}